\documentclass{article}

\usepackage{microtype}
\usepackage{graphicx}
\usepackage{booktabs} %

\usepackage{hyperref}

\usepackage[accepted]{icml2023}

\usepackage{amsmath}
\usepackage{amssymb}
\usepackage{mathtools}
\usepackage{amsthm}

\usepackage[capitalize,noabbrev]{cleveref}

\theoremstyle{plain}
\newtheorem{theorem}{Theorem}[section]
\newtheorem{proposition}[theorem]{Proposition}
\newtheorem{lemma}[theorem]{Lemma}

\theoremstyle{definition}
\newtheorem{definition}[theorem]{Definition}

\theoremstyle{remark}
\newtheorem{remark}[theorem]{Remark}

\usepackage[textsize=tiny]{todonotes}

\usepackage{paperpkg}
\usepackage{soul}

\DeclareMathOperator{\GW}{\mathrm{GW}}
\newcommand{\ri}{\text{(\rom 1)}}
\newcommand{\rii}{\text{(\rom 2)}}
\newcommand{\riii}{\text{(\rom 3)}}

\icmltitlerunning{A Gromov--Wasserstein Geometric View of Spectrum-Preserving Graph Coarsening}

\begin{document}

\twocolumn[
\icmltitle{A Gromov--Wasserstein Geometric View of Spectrum-Preserving \\ Graph Coarsening}

\icmlsetsymbol{equal}{*}

\begin{icmlauthorlist}
\icmlauthor{Yifan Chen}{hkbu}
\icmlauthor{Rentian Yao}{uiuc}
\icmlauthor{Yun Yang}{uiuc}
\icmlauthor{Jie Chen}{ibm}
\end{icmlauthorlist}

\icmlaffiliation{hkbu}{Hong Kong Baptist University}
\icmlaffiliation{uiuc}{University of Illinois Urbana-Champaign}
\icmlaffiliation{ibm}{MIT-IBM Watson AI Lab, IBM Research}

\icmlcorrespondingauthor{Yifan Chen}{yifanc@comp.hkbu.edu.hk}

\icmlkeywords{graph}

\vskip 0.3in
]

\printAffiliationsAndNotice{~}  %

\begin{abstract}
Graph coarsening is a technique for solving large-scale graph problems by working on a smaller version of the original graph, and possibly interpolating the results back to the original graph. It has a long history in scientific computing and has recently gained popularity in machine learning, particularly in methods that preserve the graph spectrum. This work studies graph coarsening from a different perspective, developing a theory for preserving graph distances and proposing a method to achieve this. The geometric approach is useful when working with a collection of graphs, such as in graph classification and regression. In this study, we consider a graph as an element on a metric space equipped with the Gromov--Wasserstein (GW) distance, and bound the difference between the distance of two graphs and their coarsened versions. Minimizing this difference can be done using the popular weighted kernel $K$-means method, which improves existing spectrum-preserving methods with the proper choice of the kernel. The study includes a set of experiments to support the theory and method, including approximating the GW distance, preserving the graph spectrum, classifying graphs using spectral information, and performing regression using graph convolutional networks.
Code is available at 
\url{https://github.com/ychen-stat-ml/GW-Graph-Coarsening}.
\end{abstract}

\section{Introduction}
Modeling the complex relationship among objects by using graphs and networks is ubiquitous in scientific applications~\citep{Morris+2020}. Examples range from analysis of chemical and biological networks~\citep{debnath1991structure, helma2001predictive, dobson2003distinguishing, irwin2012zinc, sorkun2019aqsoldb}, learning and inference with social interactions~\citep{oettershagen2020temporal}, to image understanding~\citep{dwivedi2020benchmarkgnns}. Many of these tasks are faced with large graphs, the computational costs of which can be rather high even when the graph is sparse (e.g., computing a few extreme eigenvalues and eigenvectors of the graph Laplacian can be done with a linear cost by using the Lanczos method, while computing the round-trip commute times requiring the pesudoinverse of the Laplacian, which admits a cubic cost). Therefore, it is practically important to develop methods that can efficiently handle graphs as their size grows. In this work, we consider one type of methods, which resolve the scalability challenge through working on a smaller version of the graph. The production of the smaller surrogate is called \emph{graph coarsening}.

\begin{figure}[t]
  \centering
  \includegraphics[width=\columnwidth]{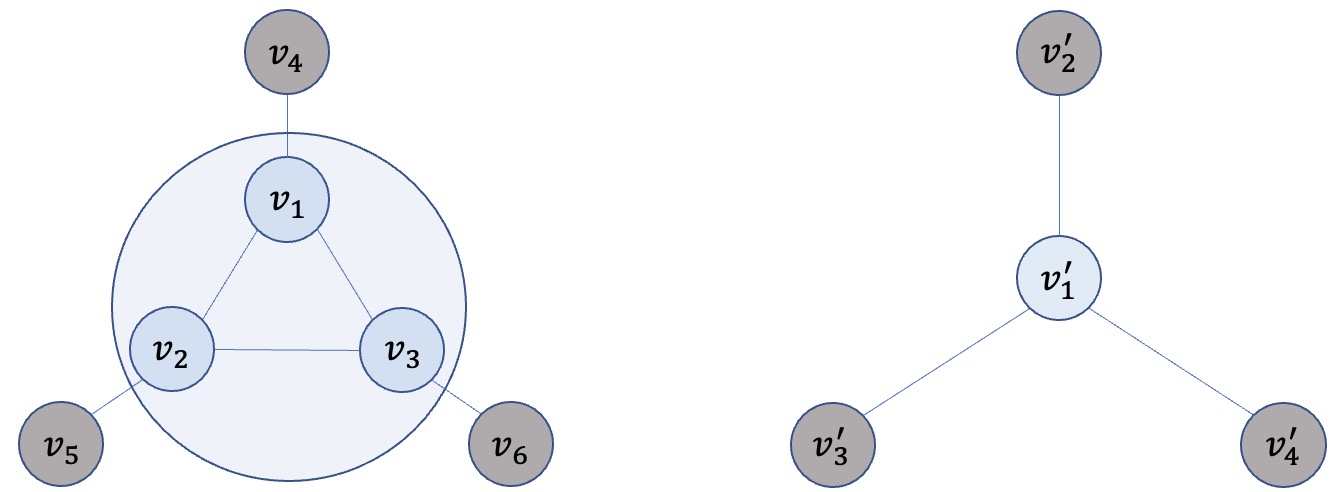}
  \caption{An example of graph coarsening. Three nodes $v_1, v_2$, and $v_3$ are merged to a supernode $v'_1$ in the coarsened graph, while each of the other nodes ($v_4$, $v_5$, and $v_6$) is a separate supernode.}
  \label{fig:coarening procedure}
\end{figure}

Graph coarsening is a methodology for solving large-scale graph problems. Depending on the problem itself, one may develop a solution based on the coarsened graph, or interpolate the solution back to the original one. \Cref{fig:coarening procedure} illustrates a toy example of graph coarsening, whereby a cluster of nodes ($v_1$, $v_2$, and $v_3$) in the original graph are merged to a so-called \emph{supernode} ($v'_1$) in the coarsened graph. 
If the problem is graph classification (predicting classes of multiple graphs in a dataset), one may train a classifier by using these smaller surrogates, if it is believed that they inherit the characteristics of the original graphs necessary for classification. 
On the other hand, if the problem is node regression,%
\footnote{Machine learning problems on graphs could be on the graph level (e.g., classifying the toxicity of a protein graph) or on the node level (e.g., predicting the coordinates of each node (atom) in a molecular graph). In this work, our experiments focus on graph-level problems.}
one may predict the targets for the nodes in the coarsened graph and interpolate them in the original graph \citep{huang2021scaling}.

Graph coarsening has a long history in scientific computing, while it gains attraction in machine learning recently~\citep{chen2022graph}. A key question to consider is what characteristics should be preserved when reducing the graph size. A majority of work in machine learning focuses on the spectral properties~\citep{loukas2018spectrally, loukas2019graph, bravo2019unifying, jin2020graph}. That is, one desires that the eigenvalues of the original graph $G$ are close to those of the coarsened graph $G^{(c)}$. While being attractive, it is unclear why this objective is effective for problems involving a collection of graphs (e.g., graph classification), where the distances between graphs shape the classifier and the decision boundary. Hence, in this work, we consider a different objective, which desires that the distance of a pair of graphs, $\text{dist}(G_1,G_2)$, is close to that of their coarsened versions, $\text{dist}(G_1^{(c)},G_2^{(c)})$.

To achieve so, we make the following contributions.

\begin{itemize}[itemsep=-0.3em, topsep=0.0em, leftmargin=1.0em]
\item We consider a graph as an element on the metric space endowed with the Gromov--Wasserstein (GW) distance~\citep{chowdhury2019gromov}, which formalizes the distance of graphs with different sizes and different node weights. We analyze the distance between $G$ and $G^{(c)}$ as a major lemma for subsequent results.

\item We establish an upper bound on the difference of the GW distance between the original graph pair and that of the coarsened pair. Interestingly, this bound depends on only the respective spectrum change of each of the original graphs. Such a finding explains the effectiveness of spectrum-preserving coarsening methods for graph classification and regression problems.

\item We bridge the connection between the upper bound and weighted kernel $K$-means, a popular clustering method, under a proper choice of the kernel. This connection leads to a graph coarsening method that we demonstrate to exhibit attractive inference qualities for graph-level tasks, compared with other spectrum-preserving methods.
\end{itemize}

\section{Related Work}
\label{sec:related}

Graph coarsening was made popular by scientific computing many decades ago, where a major problem was to solve large, sparse linear systems of equations~\citep{saad2003}. (Geometric) Multigrid methods~\citep{Briggs2000} were used to solve partial differential equations, the discretization of which leads to a mesh and an associated linear system. Multigrid methods solve a smaller system on the coarsened mesh and interpolate the solution back to the original mesh. When the linear system is not associated with a geometric mesh, algebraic multigrid methods~\citep{ruge1987algebraic} were developed to treat the coefficient matrix as a general graph, so that graph coarsening results in a smaller graph and the procedure of ``solving a smaller problem then interpolating the solution on the larger one'' remains applicable.

Graph coarsening was introduced to machine learning as a methodology of graph summarization~\citep{10.1145/3186727} through the concepts of ``graph cuts,'' ``graph clustering,'' and ``graph partitioning.'' A commonality of these concepts is that graph nodes are grouped together based on a certain objective. Graph coarsening plays an important role in multilevel graph partitioning~\citep{Karypis1999}. The normalized cut~\citep{shi2000normalized} is a well-known and pioneering method for segmenting an image treated as a graph. \citet{dhillon2007weighted} compute weighted graph cuts by performing clustering on the coarsest graph resulting from hierarchical coarsening and refining the clustering along the reverse hierarchy. Graph partitioning is used to form convolutional features in graph neural networks~\citep{defferrard2016convolutional} and to perform neighborhood pooling~\citep{Ying2018}. Graph coarsening is used to learn node embeddings in a hierarchical manner~\citep{Chen2018}. For a survey of graph coarsening with comprehensive accounts on scientific computing and machine learning, see~\citet{chen2022graph}.

A class of graph coarsening methods aim to preserve the spectra of the original graphs. 
\citet{loukas2018spectrally} and \citet{loukas2019graph} introduce the notion of ``restricted spectral similarity'' (RSS), requiring the eigenvalues and eigenvectors of the coarsened graph Laplacian, when restricted to the principal eigen-subspace, to approximate those of the original graph. Local variation algorithms are developed therein to achieve RSS. \citet{jin2020graph} suggest the use of a spectral distance as the key metric for measuring the preservation of spectral properties. The authors develop two coarsening algorithms to maximally reduce the spectral distance between the original and the coarsened graph. \citet{bravo2019unifying} develop a probabilistic algorithm to coarsen a graph while preserving the Laplacian pseudoinverse, by using an unbiased procedure to minimize the variance.

Graph coarsening is increasingly used in deep learning. One limitation of traditional coarsening methods is that they mean to be universally applicable to any graphs, without the flexibility of adapting to a particular dataset or distribution of its own characteristics. \citet{DBLP:conf/iclr/CaiWW21} address this limitation by adjusting the edge weights in the coarsened graph through graph neural networks (GNNs). Conversely, graph coarsening techniques can be applied to scale up GNNs by preprocessing the graphs~\citep{huang2021scaling}. On a separate note, graph condensation~\citep{jin2021graph} is another technique to accelerate GNNs, which uses supervised learning to condense node features and the graph structure.
Furthermore, graph pooling can assign a node in the original graph to multiple supernodes~\citep{grattarola2022understanding}, similar to the relaxation-based graph coarsening strategy explored in the algebraic multigrid literautre~\citep{Dorit2011}.

\section{Notations and Preliminaries}

In this section, we set up the notations for graph coarsening and review the background of Gromov--Wasserstein distance. Additional information can be found in Appendix~\ref{app: list_notation}.

\subsection{Graphs and Laplacians}

We denote an undirected graph as $G$, whose node set is $\m V \defeq \set{v_i}_{i = 1}^N$ with size $N = |\m V|$ and whose symmetric weighted adjacency is $\mtx{A} \defeq [a_{ij}]$. The $N$-by-$N$ (combinatorial) Laplacian matrix is defined as $\mtx{L} \defeq \mtx{D} - \mtx{A}$, where $\mtx{D} \defeq \text{diag}\paren{\mtx{A 1}_N}$ is the degree matrix. The normalized Lapalcian is $\mtx{\m L} \defeq \mtx{D}^{-\frac12} \mtx{L} \mtx{D}^{-\frac12}$. Without loss of generality, we assume that $G$ is connected, in which case the smallest eigenvalue of $\mtx{L}$ and $\mtx{\m L}$ is zero and is simple.

\subsection{Graph coarsening and coarsened graphs}
\label{sec:gc_def}

Given a graph $G$, graph coarsening amounts to finding a smaller graph $G^{(c)}$ with $n \leq |\m V|$ nodes to approximate $G$. One common coarsening approach obtains the coarsened graph from a parititioning $\m P = \set{\m P_1, \m P_2, \dots, \m P_n}$ of the node set $\m V$~\citep{loukas2018spectrally}. In this approach, each subset $\m P_j$ of nodes are collapsed to a supernode in the coarsened graph and the edge weight between two supernodes is the sum of the edge weights crossing the two corresponding partitions. Additionally, the sum of the edge weights in a partition becomes the weight of the supernode. In the matrix notation, we use $\mtx C_p \in \set{0, 1}^{n \times N}$ to denote the membership matrix induced by the partitioning $\m P$, with the $(k,i)$ entry being $\mtx C_p(k, i) = 1_{\{v_i\in\m P_k\}}$, where $1_{\{\cdot\}}$ is the indicator function. For notational convenience, we define the adjacency matrix of the coarsened graph to be $\mtx{A}^{(c)}=\mtx C_p \mtx A \mtx C_p^T$. Note that $\mtx A^{(c)}$ includes not only edge weights but also node weights.

If we similarly define $\mtx{D}^{(c)} \defeq \text{diag}\paren{\mtx{A^{(c)} 1}_n}$ to be the degree matrix of the coarsened graph $G^{(c)}$, then it can be verified that the matrix $\mtx L^{(c)}=\mtx{C}_p \mtx{L} \mtx{C}_p^{\intercal} = \mtx{D}^{(c)} - \mtx A^{(c)}$ is a (combinatorial) Laplacian (because its smallest eigenvalue is zero and is simple). Additionally, $\mtx{\m L}^{(c)} = \paren{\mtx{D}^{(c)}}^{-\frac12} \mtx{L}^{(c)} \paren{\mtx{D}^{(c)}}^{-\frac12}$ is the normalized Laplacian.

In the literature, the objective of coarsening is often to minimize some difference between the original and the coarsened graph. For example, in spectral graph coarsening, \citet{loukas2019graph} proposes that the $\mtx L$-norm of any $N$-dimensional vector $\mtx x$ be close to the $\mtx L^{(c)}$-norm of the $n$-dimensional vector $(\mtx{C}_p^{\intercal})^+ \mtx x$; while \citet{jin2020graph} propose to minimize the difference between the ordered eigenvalues of the Laplacian of $G$ and those of the lifted Laplacian of $G^{(c)}$ (such that the number of eigenvalues matches). Such objectives, while being natural and interesting in their own right, are not the only choice. Questions may be raised; for example, it is unclear why the objective is to preserve the Laplacian spectrum but not the degree distribution or the clustering coefficients. Moreover, it is unclear how preserving the spectrum benefits the downstream use. In this paper, we consider a different objective---preserving the graph distance---which may help, for example, maintain the decision boundary in graph classification problems. To this end, we review the Gromov--Wasserstein (GW) distance.

\subsection{Gromov--Wasserstein distance and its induced metric space}

The GW distance was originally proposed by \citet{memoli2011gromov} to measure the distance between two metric measure spaces $M_{\m X}$ and $M_{\m Y}$.%
\footnote{A metric measure space $M_{\m X}$ is the triple $(\m X, d_X, \mu_X)$, where $(\m X, d_X)$ is a metric space with metric $d_X$ and $\mu_X$ is a Borel probability measure on $\m X$.} However, using metrics to characterize the difference between elements in sets $\mathcal{X}$ and $\mathcal{Y}$ can be too restrictive. \citet{peyre2016gromov} relaxed the metric notion by proposing the GW discrepancy, which uses dissimilarity, instead of metric, to characterize differences. \citet{chowdhury2019gromov} then extended the concept of GW distance/discrepancy from metric measure spaces to measure networks, which can be considered a generalization of graphs.
\begin{definition}[Measure network]
\label{def:measure networks}
A measure network is a triple $\paren{\m X, \omega_X , \mu_X}$, where $\m X$ is a Polish space (a separable and completely metrizable topological space), $\mu_X$ is a fully supported Borel probability measure, and $\omega_X$ is a bounded measurable function on $\m X^2$.
\end{definition}

In particular, a graph $G$ (augmented with additional information) can be taken as a discrete measure network. We let $\m X$ be the set of graph nodes $\set{v_i}_{i = 1}^N$ and associate with it a probability mass $\mu_X = \mtx m = \brkt{m_1, \dots,m_{N}}^{\intercal} \in \mb R_+^{N}$, $\sum_{i=1}^{N} m_i=1$. Additionally, we associate $G$ with the node similarity matrix $\mtx S = [s_{ij}] \in \mb R^{N \times N}$, whose entries are induced from the measurable map $\omega_X$: $s_{ij} = \omega_X\paren{v_i, v_j}$. Note that the mass $\mtx m$ and the similarity $\mtx S$ do not necessarily need to be related to the node weights and edge weights, although later we will justify and advocate the use of some variant of the graph Lapalcian as $\mtx S$.

For a source graph $G_s$ with $N_s$ nodes and mass $\mtx m_s$ and a target graph $G_t$ with $N_t$ nodes and mass $\mtx m_t$, we can define a transport matrix $\mtx{T} = [t_{ij}] \in\mb R^{N_s\times N_t}$, where $t_{ij}$ specifies the probability mass transported from $v_i^s$ (the $i$-th node of $G_s$) to $v_j^t$ (the $j$-th node of $G_t$). We denote the collection of all feasible transport matrices as $\Pi_{s,t}$, which includes all $\mtx T$ that satisfy $\mtx{T1}=\mtx m_s$ and $\mtx{T}^\intercal \mtx 1 = \mtx m_t$. Using the $\ell_p$ transportation cost $L(a, b)=(a-b)^p$, \citet{chowdhury2019gromov} define the $\text{GW}_p$ distance for graphs as
\begin{align}
\label{eqn:gwd}
\GW_p^p(G_s, G_t) 
=& \min_{\mtx T \in \Pi_{s, t}} \sum_{i, j=1}^{N_S}\sum_{i',j'=1}^{N_t} \left|s^s_{ij} - s^t_{i'j'} \right|^p \mtx T_{ii'} \mtx T_{jj'} \nonumber \\
=& \min_{\mtx T \in \Pi_{s, t}} \dotp{\mtx M}{\mtx T},
\end{align}
whre the cross-graph dissimilarity matrix $\mtx M \in \mb R^{N_s \times N_t}$ has entries $\mtx M_{jj'} = \sum_{i, i'} \left|s^s_{ij} - s^t_{i'j'} \right|^p \mtx T_{ii'}$ (which by themselves are dependent on $\mtx T$).

The computation of the $\text{GW}_p$ distance (GW distance for short) can be thought of as finding a proper alignment of nodes in two graphs, such that the aligned nodes $v^s_j$ and $v^t_{j'}$ have similar interactions with other nodes in their respective graphs. Intuitively, this is achieved by assigning large transport mass $\mtx T_{jj'}$ to a node pair $(v^s_j, v^t_{j'})$ with small dissimilarity $\mtx M_{jj'}$, making the GW distance a useful tool for graph matching~\citep{xu2019gromov}. We note that variants of the GW distance exist; for example, \citet{titouan2019optimal} proposed a fused GW distance by additioanlly taking into account the dissimilarity of node features. We also note that computing the GW distance is NP-hard, but several approximate methods were developed~\citep{xu2019scalable,zheng2022brief}. In this work, we use the GW distance as a theoretical tool and may not need to compute it in action.

On closing this section, we remark that \citet[Theorem 18]{chowdhury2019gromov} show that the GW distance is indeed a metric for measure networks, modulo weak isomorphism.%
\footnote{The weak isomorphism allows a node to be split into several identical nodes with the mass preserved. See Definition~3 of \citet{chowdhury2019gromov}.}
Therefore, we can formally establish the metric space of interest.
\begin{definition}[$\GW_p$ space]
Let $\m N$ be the collection of all measure networks. For $p \geq 1$, we denote by $\paren{\m N, \GW_p}$ the metric space of measure networks endowed with the $\mathrm{GW}_p$ distance defined in~\eqref{eqn:gwd} and call it the $\GW_p$ space.
\end{definition}

\section{Graph Coarsening from a Gromov--Wasserstein Geometric View}

In this section, we examine how the GW geometric perspective can shape graph coarsening. In \Cref{sec:scale}, we provide a framework that unifies many variants of the coarsening matrices through the use of the probability mass $\mtx m$ introduced in the context of measure networks. Then, in \Cref{sec:gwd}, we analyze the variant associated with the similarity matrix $\mtx S$. In particular, we establish an upper bound of the difference of the GW distances before and after coarsening. Based on the upper bound, in \Cref{sec:svdmap} we connect it with some spectral graph techniques and in \Cref{sec:similarity} we advocate a choice of $\mtx S$ in practice.

\subsection{A unified framework for coarsening matrices}\label{sec:scale}

The membership matrix $\mtx C_p$ is a kind of coarsening matrices: it connects the adjacency matrix $\mtx A$ with the coarsened version $\mtx A^{(c)}$ through $\mtx A^{(c)} = \mtx C_p \mtx A \mtx C_p^{\intercal}$. There are, however, different variants of coarsening matrices. We consider three here, all having a size $n \times N$.

\begin{enumerate}[itemsep=-0.3em, topsep=0.0em, leftmargin=1.7em, label=(\roman*)]
\item Accumulation. This is the matrix $\mtx C_p$. When multiplied to the left of $\mtx A$, the $i$-th row of the product is the sum of all rows of $\mtx A$ corresponding to the partition $\m P_i$.

\item Averaging. A natural alternative to summation is (weighted) averaging. We define the diagonal mass matrix $\mtx W = \text{diag}\paren{m_1, \cdots, m_N}$ and for each partition $\m P_i$, the accumulated mass $c_i = \sum_{j\in\m P_i}m_j$ for all $i\in[n]$. Then, the averaging coarsening matrix is
$$\bar{\mtx C}_w \defeq \mathrm{diag}(c_1^{-1}, \cdots, c_n^{-1}) \mtx C_p \mtx W.$$
This matrix takes the effect of averaging because of the division over $c_i$. Moreover, when the probability mass $\mtx m$ is uniform (i.e., all $m_j$'s are the same), we have the relation $\bar{\mtx C}_w^+ = \mtx C_p^{\intercal}$.

\item Projection. Neither $\mtx C_p$ nor $\bar{\mtx C}_w$ is orthogonal. We define the projection coarsening matrix as
$$\mtx C_w \defeq \mathrm{diag}(c_1^{-1/2}, \cdots, c_n^{-1/2}) \mtx C_p \mtx W^\frac12,$$
by noting that $\mtx C_w \mtx C_w^\intercal$ is the identity and hence $\mtx C_w$ has orthonoral rows. Therefore, the $N$-by-$N$ matrix $\mtx \Pi_w \defeq \mtx W^{\frac12} \mtx C_p^\intercal \bar{\mtx C}_w \mtx W^{-\frac12} = \mtx C_w^\intercal \mtx C_w$ is a projection operator.
\end{enumerate}

In \Cref{sec:gc_def}, we have seen that the combinatorial Laplacian $\mtx L$ takes $\mtx C_p$ as the coarsening matrix, because the relationship $\mtx L^{(c)} = \mtx C_p \mtx L \mtx C_p^{\intercal}$ inherits from $\mtx A^{(c)} = \mtx C_p \mtx A \mtx C_p^{\intercal}$. On the other hand, it can be proved that, if we take the diagonal mass matrix $\mtx W$ to be the degree matrix $\mtx D$, the normalized Laplacian defined in \Cref{sec:gc_def} can be written as $\mtx{\m L}^{(c)} = \mtx C_w \mtx{\m L} \mtx C_w^\intercal$ (see \Cref{app: doubly}). In other words, the normalized Laplacian uses $\mtx C_w$ as the coarsening matrix. The matrix $\mtx{\m L}^{(c)}$ is called a doubly-weighted Laplacian~\citep{chung1996combinatorial}.

For a general similarity matrix $\mtx S$ (not necessarily a Laplacian), we use the averaging coarsening matrix $\bar{\mtx{C}}_w$ and define $\mtx S^{(c)} \defeq \bar{\mtx{C}}_w \mtx{S} \bar{\mtx{C}}_w^{\intercal}$. This definition appears to be more natural in the GW setting; see a toy example in \Cref{app: toy scale}. It is interesting to note that \citet{vincent-cuaz2022semirelaxed} proposed the concept of semi-relaxed GW (srGW) divergence, wherein the first-order optimality condition of the constrained srGW barycenter problem is exactly the equality $\mtx S^{(c)} = \bar{\mtx{C}}_w \mtx{S} \bar{\mtx{C}}_w^{\intercal}$. See \Cref{app:srgw}.

In the next subsection, we will consider the matrix $\mtx U \defeq \mtx W^\frac{1}{2}\mtx S\mtx W^\frac{1}{2}$. We define the coarsened version by using the projection coarsening matrix $\mtx C_w$ as in $\mtx U^{(c)} \defeq \mtx C_w \mtx U \mtx C_w^{\intercal}$. We will bound the distance of the original and the coarsened graph by using the eigenvalues of $\mtx U$ and $\mtx U^{(c)}$. The reason why $\mtx U$ and $\mtx S$ use different coarsening matrices lies in the technical subtlety of $\mtx W^\frac{1}{2}$: $\mtx C_w$ absorbs this factor from $\bar{\mtx{C}}_w$.

\subsection{Graph distance on the \texorpdfstring{$\text{GW}_2$}{GW-2} space}
\label{sec:gwd}

Now we consider the GW distance between two graphs. For theoretical and computational convenience, we take the $\ell_2$ transportation cost (i.e., taking $p=2$ in~\eqref{eqn:gwd}). When two graphs $G_1$ and $G_2$ are concerned, we inherit the subscripts $_1$ and $_2$ to all related quantities, such as the probability masses $\mtx m_1$ and $\mtx m_2$, avoiding verbatim redundancy when introducing notations. This pair of subscripts should not cause confusion with other subscripts, when interpreted in the proper context. Our analysis can be generalized from the $\text{GW}_2$ distance to other distances, as long as the transportation cost satisfies the following decomposable condition.

\begin{proposition}{\citep{peyre2016gromov}}\label{prop: decompose_loss}
Let $\mtx T^\ast\in \mb R^{N_1\times N_2}$ be the optimal transport plan from $\mtx m_1$ to $\mtx m_2$,%
\footnote{The existence of $\mtx T^*$ is guaranteed by the fact that the feasible region of $\mtx T$ is compact and the object function $\dotp{\mtx M}{\mtx T}$ is continuous with respect to $\mtx T$.}
and $\mtx S_k \in \mb R^{N_k \times N_k}$ be similarity matrices for $k=1, 2$. If the transport cost can be written as $L(a, b)=f_{1}(a) + f_{2}(b) - h_{1}(a) h_{2}(b)$ for some element-wise functions $\left(f_{1}, f_{2}, h_{1}, h_{2}\right)$, then we can write $\mtx M$ in~\eqref{eqn:gwd} as
$$
f_{1}(\mtx S_1) \mtx m_1 \mtx{1}_{N_{2}}^{\intercal} + \mtx{1}_{N_{1}} \mtx m_2^{\intercal} f_{2}(\mtx S_2)^{\intercal} - h_{1}(\mtx S_1) \mtx T^* h_{2}(\mtx S_2)^{\intercal}.
$$
\end{proposition}

Clearly, for the squared cost $L(a, b) = (a-b)^2$, we may take $f_1(a) = a^2$, $f_2(b) = b^2$, $h_1(a) = a$, and $h_2(b) = 2b$.

We start the analysis by first bounding the distance between $G$ and $G^{(c)}$.

\begin{theorem}[Single graph]
\label{thm: self-coarsened}
Consider a graph $G$ with positive semi-definite (PSD) similarity matrix $\mtx S$ and diagonal mass matrix $\mtx W$ and similarly the coarsened graph $G^{(c)}$. Let $\lambda_{1}\geq \lambda_{2}\geq  \cdots \geq \lambda_{N}$ be the sorted eigenvalues of $\mtx U = \mtx W^\frac{1}{2}\mtx S\mtx W^\frac{1}{2}$ and $\lambda^{(c)}_{1} \geq \cdots \geq \lambda^{(c)}_{n}$ be the sorted eigenvalues of $\mtx U^{(c)} = \mtx C_{w}\mtx U\mtx C_{w}^\intercal$. Then,
\begin{align}
\label{eqn:self-coarsened-bd}
\GW_2^2(G, G^{(c)}) \leq \lambda_{N-n+1} \sum_{i=1}^{n} \paren{\lambda_{i} - \lambda^{(c)}_{i}} + C_{\mtx U, n},
\end{align}
where $C_{\mtx U, n} = \sum_{i=1}^{n}\lambda_{i}(\lambda_{i} - \lambda_{N - n + i}) + \sum_{i=n+1}^{N}\lambda_{i}^2$ is non-negative and is independent of coarsening.
\end{theorem}

\begin{remark}
(\rom 1) The bound is tight when $n=N$ because the right-hand side is zero in this case.
(\rom 2) The choice of coarsening only affects the spectral difference
\begin{align}
\label{eqn:delta}
\Delta \defeq \sum_{i=1}^n(\lambda_{i} - \lambda_{i}^{(c)}),
\end{align}
because $C_{\mtx U, n}$ is independent of it. Each term $\lambda_{i} - \lambda_{i}^{(c)}$ in $\Delta$ is non-negative due to the Poincar\'e separation theorem (see \Cref{sec:interlacing}).
(\rom 3) $\Delta$ is a generalization of the spectral distance proposed by \citet{jin2020graph}, because our matrix $\mtx U$ is not necessarily the normalized Laplacian. For additional discussions, see \Cref{app:spectral_distance}.
(\rom 4) When $\mtx U$ is taken as the normalized Laplacian, our bound is advantageous over the bound established by \citet{jin2020graph} in the sense that $\Delta$ is the only term impacted by coarsening and that no assumptions on the $K$-means cost are imposed.
\end{remark}

We now bound the difference of distances. The following theorem suggests that the only terms dependent on coarsening are $\Delta_1$ and $\Delta_2$, counterparts of $\Delta$ in~\Cref{thm: self-coarsened}, for graphs $G_1$ and $G_2$ respectively.

\begin{theorem}\label{thm: coarsen12}
Given a pair of graphs $G_1$ and $G_2$, we extend all notations in \Cref{thm: self-coarsened} by adding subscripts $_1$ and $_2$ respectively for $G_1$ and $G_2$. We denote the optimal transport plan induced by $\GW_2(G_1, G_2)$ as $\mtx T^*$ and let the normalized counterpart be $\mtx P = \mtx W_1^{-\frac12} \mtx T^* \mtx W_2^{-\frac12}$. Additionally, we define $\mtx V_1 \defeq \mtx P \mtx W_2^{\frac12} \mtx S_2 \mtx W_2^{\frac12} \mtx P^\intercal$ with eigenvalues $\nu_{1,1} \geq \nu_{1,2}\geq \cdots \geq \nu_{1,N_1}$ and $\mtx V_2 \defeq \mtx P^\intercal \mtx W_1^{\frac12} \mtx S_1 \mtx W_1^{\frac12} \mtx P$ with eigenvalues $\nu_{2,1} \geq \nu_{2,2}\geq \cdots \geq \nu_{2,N_2}$, both independent of coarsening. Then, $\abs{\GW_2^2(G_1^{(c)}, G_2^{(c)}) - \GW_2^2(G_1, G_2)}$ is upper bounded by
\begin{align*}
    & \max \Big\{ \lambda_{1, N_1-n_1+1} \cdot \Delta_1 + C_{\mtx U_1, n_1} \\
&\qquad\qquad + \lambda_{2, N_2-n_2+1} \cdot \Delta_2 + C_{\mtx U_2, n_2}, \\
    &\qquad 2 \cdot \left[ \nu_{1, N_1 - n_1 + 1} \cdot \Delta_1 + C_{\mtx U_1, \mtx V_1, n_1} \right. \\
&\qquad\qquad + \left. \nu_{2, N_2 - n_2 + 1} \cdot \Delta_2 + C_{\mtx U_2, \mtx V_2, n_2} \right] \Big\},
\end{align*}
where $C_{\mtx U_1, n_1}$ is from \Cref{thm: self-coarsened} and the other coarsening-independent terms $\mtx U_2, C_{\mtx U_2, n_2}, C_{\mtx U_2, \mtx V_2, n_2}, C_{\mtx U_2, \mtx V_2, n_2}$ are introduced in \Cref{lem: bd_diff_I3_further} in Appendix~\ref{ref:gw2cg}.
\end{theorem}

\begin{remark}
(\rom 1) The above bound takes into account both the differences between two graphs and their respective coarsenings. Even when the two graphs are identical $G_1 = G_2$, the bound can still be nonzero if the coarsened graphs $G_1^{(c)}$ and $G_2^{(c)}$ do not match.
(\rom 2) The decoupling of $\Delta_1$ and $\Delta_2$ offers an algorithmic benefit when one wants to optimize the differences of distances for all graph pairs in a dataset: it suffices to optimize the distance between $G$ and $G^{(c)}$ for each graph individually. This benefit is in line with the prior practice of directly applying spectrum-preserving coarsening methods for graph-level tasks \citep{jin2020graph, huang2021scaling}. Their experimental results and our numerical verification in \Cref{sec:exp} show that our bound is useful and it partly explains the empirical success of spectral graph coarsening.
\end{remark}

\subsection{Connections with truncated SVDs and Laplacian eigenmaps}
\label{sec:svdmap}

The spectral difference $\Delta=\sum_{i=1}^n \paren{\lambda_i - \lambda^{(c)}_i}$ in \Cref{thm: self-coarsened} can be used as the loss function for defining an optimal coarsening. Because $\Delta + \sum_{i=n+1}^N \lambda_i = \Tr(\mtx U) - \Tr(\mtx C_w \mtx U \mtx C_w^{\intercal})$ and because $\sum_{i=n+1}^N \lambda_i$ is independent of coarsening, minimizing $\Delta$ is equivalent to the following problem
\begin{align} \label{eqn:trace_loss}
\min_{{\mtx{C}_w}} \Tr\paren{\mtx U - \mtx\Pi_w \mtx U \mtx\Pi_w},
\end{align}
by recalling that $\mtx\Pi_w = \mtx C_w^{\intercal} \mtx C_w$ is a projector. The problem~\eqref{eqn:trace_loss} is a well-known trace optimization problem, which has rich connections with many spectral graph techniques~\citep{Kokiopoulou2011}.

\textbf{Connection with truncated SVD.}
At first sight, a small trace difference does not necessarily imply the two matrices are close. However, because $\mtx \Pi_w$ is a projector, the Poincar\'e separation theorem (see \Cref{sec:interlacing}) suggests that their eigenvalues can be close. A well-known example of using the trace to find optimal approximations is the truncated SVD, which retains the top singular values (equivalently eigenvalues for PSD matrices). The truncated SVD is a technique to find the optimal rank-$n$ approximation of a general matrix $\mtx U$ in terms of the spectral norm or the Frobenius norm, solving the problem
\begin{align*}
\min_{\mtx{C} \in \m O_n} \Tr\left(\mtx U - \mtx{C}^{\intercal} \mtx{C} \mtx{U} \mtx{C}^{\intercal} \mtx{C} \right),
\end{align*}
where $\m O_n$ is the class of all rank-$n$ orthogonal matrices. The projection coarsening matrix $\mtx C_w$ belongs to this class.

\textbf{Connection with Laplacian eigenmaps.}
The Laplacian eigenmap~\citep{belkin2003laplacian} is a manifold learning technique that learns an $n$-dimensional embedding for $N$ points connected by a graph. The embedding matrix $\mtx Y$ solves the trace problem $\min_{\mtx{Y} \mtx{D} \mtx{Y}^\intercal = \mtx I_n} \Tr \left(\mtx{Y} \mtx{L} \mtx{Y}^{\intercal} \right)$. If we let $\bar{\mtx Y} = \mtx Y \mtx D^{-\frac{1}{2}}$, the problem is equivalent to
\begin{align*}
\min_{\bar{\mtx{Y}} \in \m O_n} \Tr\left(\bar{\mtx{Y}} \mtx{\m L} \bar{\mtx{Y}}^{\intercal} \right).
\end{align*}
Different from truncated SVD, which uses the top singular vectors (eigenvectors) to form the solution, the Laplacian eigenmap uses the bottom eigenvectors of $\mtx{\m L}$ to form the solution.

\subsection{Signless Laplacians as similarity matrices}
\label{sec:similarity}

The theory established in \Cref{sec:gwd} is applicable to any PSD matrix $\mtx S$, but for practical uses we still have to define it. Based on the foregoing exposition, it is tempting to let $\mtx S$ be the Laplacian $\mtx L$ or the normalized Laplacian $\mtx{\m L}$, because they are PSD and they reveal important information of the graph structure~\citep{tsitsulin2018netlsd}. In fact, as a real example, \citet{chowdhury2021generalized} used the heat kernel as the similarity matrix to define the GW distance (and in this case the GW framework is related to spectral clustering). However, there are two problems that make such a choice troublesome. First, the (normalized) Laplacian is sparse but its off-diagonal, nonzero entries are negative. Thus, when $\mtx S$ is interpreted as a similarity matrix, a pair of nodes not connected by an edge becomes more similar than a pair of connected nodes, causing a dilema. Second, under the trace optimization framework, one is lured to intuitively look for solutions toward the bottom eigenvectors of $\mtx S$, like in Laplacian eigenmaps, an opposite direction to the true solution of~\eqref{eqn:trace_loss}, which is toward the top eigenvectors instead.

To resolve these problems, we propose to use the signless Laplacian~\citep{cvetkovic2007signless}, $\mtx D + \mtx A$, or its normalized version, $\mtx{I}_N + \mtx{D}^{-\frac12} \mtx{A} \mtx{D}^{-\frac12}$, as the similarity matrix $\mtx S$. These matrices are PSD and their nonzero entries are all positive. With such a choice, the spectral difference $\Delta$ in~\eqref{eqn:delta} has a fundamentally different behavior from the spectral distance used by \citet{jin2020graph} when defining the coarsening objective.

\section{Computational Equivalence between Graph Coarsening and Weighted Kernel \texorpdfstring{$K$}{K}-means}
\label{sec:equivalence}

By recalling that $\mtx U = \mtx W^{\frac12} \mtx S \mtx W^{\frac12}$ and that $\mtx \Pi_w = \mtx{C}_w^{\intercal} \mtx{C}_w$ is a projector, we rewrite the coarsening objective in~\eqref{eqn:trace_loss} as
\begin{align}\label{eqn:kmeans}
\Tr\paren{\mtx W^{\frac12} \mtx S \mtx W^{\frac12}} - \Tr\paren{\mtx{C}_w \mtx W^{\frac12} \mtx S \mtx W^{\frac12} \mtx{C}_w^{\intercal}}.
\end{align}
When $\mtx S$ is PSD, it could be interpreted as a kernel matrix such that there exists a set of feature vectors $\{\mtx \phi_i\}$ for which $\mtx S_{ij} = \langle \mtx \phi_i, \mtx \phi_j \rangle$ for $i,j = 1,\ldots,N$. Then, with simple algebraic manipulation, we see that~\eqref{eqn:kmeans} is equivalent to the well-known clustering objective:
\begin{equation}\label{eqn:kmeans_obj}
\sum_k \sum_{i \in \m P_k} m_i \| \mtx \phi_i - \mtx \mu_k \|^2
\quad\text{with}\quad
\mtx \mu_k = \sum_{i \in \m P_k} \frac{m_i}{c_k} \mtx \phi_i ,
\end{equation}
where the norm $\|\cdot\|$ is induced by the inner product $\langle \cdot, \cdot \rangle$. Here, $\mtx \mu_k$ is the weighted center of $\mtx \phi_i$ for all nodes $i$ belonging to cluster $k$. Hence, the weighted kernel $K$-means algorithm~\citep{dhillon2004kernel, dhillon2007weighted} can be applied to minimize~\eqref{eqn:kmeans_obj} by iteratively recomputing the centers and updating cluster assignments according to the distance of a node to all centers. We denote the squared distance between $\mtx \phi_i$ and any $\mtx \mu_j$ by $\mathrm{dist}^2_j(i)$, which is
\begin{align}
\mtx{S}_{ii} - 2 \underset{k \in \m P_j}{\sum} m_k \mtx{S}_{ki}/{c_j} 
+ \sum_{k_1, k_2 \in \m P_j} m_{k_1} m_{k_2} \mtx{S}_{k_1 k_2}/{c_j^2}.
\label{eqn:kmeans_dist}
\end{align}
The $K$-means algorithm is summarized in \Cref{alg:WKKmeans} and we call this method \emph{kernel graph coarsening} (KGC).

\begin{algorithm}[tb]
  \caption{Kernel graph coarsening (KGC).}\label{alg:WKKmeans}
    {\bfseries Input:} $\mtx S$: similarity matrix, $\mtx m$: node mass vector, $n$: number of clusters \\
    {\bfseries Output:} $\m P = \set{\m P_i}_{i=1}^n$: node partition \\
    \vspace{-1em}
  \begin{algorithmic}[1] %
    \FUNCTION[$\mtx S, \mtx m, n$]{KGC}
    \STATE Initialize the $n$ clusters: $\m P^{(0)} = \set{\m P_1^{(0)}, \dots, \m P_n^{(0)}}$; $c^{(0)}_j = \sum_{k \in \m P^{(0)}_j} m_k, \forall j \in [n]$.
    \STATE Set the counter of iterations $t=0$.
    \FOR{node $i=1$ {\bfseries to} $N$}
    \STATE Find its new cluster index by~\eqref{eqn:kmeans_dist}:
    \begin{align*}
        \mathrm{idx}(i) = \argmin_{j \in [n]} \mathrm{dist}^{(t)}_j(i).
    \end{align*}
    \ENDFOR
    \STATE Update the clusters: for all $j \in [n]$,
    \begin{align*}
        \m P_j^{(t+1)} = \set{i: \mathrm{idx}(i) = j}, c^{(t+1)}_j = \sum_{k \in \m P^{(t+1)}_j} m_k.
    \end{align*}
    \IF{the partition $\m P^{(t+1)}$ is invariant} 
        \STATE \textbf{return} $\m P^{(t+1)}$
    \ELSE
        \STATE Set $t = t+1$ and go to Line 4.
    \ENDIF
    \ENDFUNCTION
  \end{algorithmic}
\end{algorithm}

\textbf{KGC as a post-refinement of graph coarsening.}
KGC can be used as a standalone coarsening method. A potential drawback of this usage is that the $K$-means algorithm is subject to initialization and the clustering quality varies significantly sometimes. Even advanced initialization techniques, such as \texttt{$K$-means++}~\citep{arthur2006k}, are not gauranteed to work well in practice. Moreover, KGC, in the vanilla form, does not fully utilize the graph information (such as node features), unless additional engineering of the similarity matrix $\mtx S$ is conducted. Faced with these drawbacks, we suggest a simple alternative: initialize KGC by the output of another coarsening method and use KGC to improve it~\citep{scrucca2015improved}. In our experience, KGC almost always monotonically reduces the spectral difference $\Delta$ and improves the quality of the initial coarsening.

\textbf{Time complexity.}
Let $T$ be the number of $K$-means iterations. The time complexity of KGC is $\m O\paren{T \paren{M+Nn}}$, where $M = \mathrm{nnz}(\mtx S)$ is the number of nonzeros in $\mtx S$. See \Cref{sec:alg time} for the derivation of this cost and a comparison with the cost of spectral graph coarsening~\citep{jin2020graph}.

\section{Numerical Experiments}
\label{sec:exp}

We evaluate graph coarsening methods, including ours, on eight benchmark graph datasets: MUTAG~\citep{debnath1991structure, kriege2012subgraph}, PTC~\citep{helma2001predictive}, PROTEINS~\citep{borgwardt2005protein, schomburg2004brenda}, MSRC~\citep{neumann2016propagation}, IMDB~\citep{yanardag2015deep}, Tumblr~\citep{oettershagen2020temporal}, AQSOL~\citep{sorkun2019aqsoldb, dwivedi2020benchmarkgnns}, and ZINC~\citep{irwin2012zinc}. Information of these datasets is summarized in Table~\ref{tab:dataset}.

\begin{table}[h]
\centering
\caption{Summary of datasets. $\abs{V}$ and $\abs{E}$ denote the average number of nodes and edges, respectively. All graphs are treated undirected. $\mathrm{R}(1)$ represents a regression task.}
\label{tab:dataset}
\resizebox{0.9\columnwidth}{!}{
\begin{tabular}{ccccc}
\textbf{Dataset}   & \textbf{Classes} &\textbf{Size}  & $\mathbf{\abs{V}}$     & $\mathbf{\abs{E}}$    \\ 
\hline
MUTAG     & 2  & 188   &17.93 & 19.79 \\
PTC       & 2    & 344     &14.29 & 14.69 \\
PROTEINS  & 2 & 1113   &39.06 & 72.82 \\
MSRC     & 8 & 221   &39.31 & 77.35 \\
IMDB     & 2 & 1000   &19.77 & 96.53 \\
Tumblr    & 2  & 373   &53.11 & 199.78 \\
AQSOL    & $\mathrm{R}(1)$  & 9823   &17.57 & 17.86 \\
ZINC    & $\mathrm{R}(1)$  & 12000   &23.16 & 49.83 \\
\hline
\end{tabular}
}
\end{table}

We compare our method with the following baseline methods~\citep{loukas2019graph, jin2020graph}: \circled{1} \textbf{Variation Neighborhood Graph Coarsening (VNGC)}; \circled{2} \textbf{Variation Edge Graph Coarsening (VEGC)}; \circled{3} \textbf{Multilevel Graph Coarsening (MGC)}; and \circled{4} \textbf{Spectral Graph Coarsening (SGC)}. For our method, we consider the vanilla KGC, which uses \texttt{$K$-means++} for initialization, and the variant KGC(A), which takes the output of the best-performing baseline method for initialization. More implementation details are provided in Appendix~\ref{app: exp}.

\subsection{GW distance approximation}
\label{sec:exp_distMtx}

For a sanity check, we evaluate each coarsening method on the approximation of the GW distance, to support our motivation of minimizing the upper bound in Theorem~\ref{thm: self-coarsened}.

We first compare the average squared GW distance, by using the normalized signless Laplacian $\mtx{I}_N + \mtx{D}^{-\frac12} \mtx{A} \mtx{D}^{-\frac12}$ as the similarity matrix $\mtx{S}$ (see Section~\ref{sec:similarity}). We vary the coarsened graph size $n = \lceil c * N \rceil$ for $c = 0.3, \cdots, 0.9$. For each graph in PTC and IMDB, we compute $\GW_2^2(G, G^{(c)})$ and report the average in Table~\ref{tab:avg_gw}, Appendix~\ref{app:gw_distMtx}. We obtain similar observations across the two datasets. In particular, KGC and KGC(A) outperform baselines. When $c$ is small, it is harder for $K$-means clustering to find the best clustering plan and hence KGC(A) works better. When $c$ gets larger, KGC can work better, probably because the baseline outputs are poor intiializations.

We then report the average gap between the left- and right-hand sides of the bound~\eqref{eqn:self-coarsened-bd} in Table~\ref{tab:avg_gap}, Appendix~\ref{app:gw_distMtx}. For all methods, including the baselines, the gap is comparable to the actual squared distance shown in Table~\ref{tab:avg_gw}, showcasing the quality of the bound. Moreover, the gap decreases when $c$ increases, as expected.

We next compute the matrix of GW distances, $\mtx Z$ (before coarsening) and $\mtx Z^{(c)}$ (after coarsening), and compute the change $\|\mtx Z - \mtx Z^{(c)}\|_F$. Following previous works \citep{pmlr-v32-chan14,Xu2020LearningGV}, we set $n = \lfloor N_{\mathrm{max}} / \log(N_{\mathrm{max}}) \rfloor$. The similarity matrix for the coarsened graph uses the averaging coarsening matrix as advocated in Section~\ref{sec:scale}; that is, $\mtx S^{(c)} = \bar{\mtx C}_w \paren{\mtx{I}_N + \mtx{D}^{-\frac12} \mtx{A} \mtx{D}^{-\frac12}} \bar{\mtx C}_w^\intercal$. Additionally, we use a variant of the similarity matrix, $\mtx S^{(c)} = \mtx I_n + \paren{\mtx D^{(c)}}^{-\frac12} \mtx A^{(c)} \paren{\mtx D^{(c)}}^{-\frac12}$, resulting from the projection coarsening matrix, for comparison.

\begin{table}[t]
\caption{Change of the matrix of GW distances (computed as Frobenius norm error) and coarsening time. Dataset: PTC. Results are averged over 10 runs. MGC is a deterministic method and therefore standard deviation is 0. KGC(A) is initialized with the MGC output.}
\label{tab:gwDistMtx}
\begin{center}
\resizebox{\columnwidth}{!}{
\begin{tabular}{c|ccc}
\textbf{Coars.\ Mat.} &\textbf{Methods}  &\textbf{Frob.\ Error}$\downarrow$ &\textbf{Time}$\downarrow$ \\
\hline
\multirow{4}{*}{Projection} &VNGC             & $182.85\pm0.02$  & $6.38\pm0.01$ \\
&VEGC             & $54.81\pm0.02$  & $3.81\pm0.$ \\
&MGC             & $13.69\pm0.$  & $6.71\pm0.01$ \\
&SGC             & $12.41\pm0.04$  & $30.24\pm0.07$ \\
\hline
\multirow{6}{*}{Averaging} &VNGC             & $17.34\pm0.01$  & $6.55\pm0.18$ \\
&VEGC             & $9.22\pm0.02$  & $3.75\pm0.01$ \\
&MGC             & $5.31\pm0.$  & $6.59\pm0.02$ \\
&SGC             & $6.06\pm0.02$  & $28.06\pm0.10$ \\
&KGC             & $\mathbf{4.45\pm0.03}$  & $\mathbf{1.34\pm0.33}$ \\
&KGC(A)             & $5.28\pm0.$  & $\mathbf{0.27\pm0.}$ \\
\hline
\end{tabular}
}
\end{center}
\end{table}

Table~\ref{tab:gwDistMtx} summarizes the results for PTC. It confirms that using ``Averaging'' is better than using ``Projection'' for the coarsening matrix. Additionally, KGC(A) initialized with MGC (the best baseline) induces a significantly small change (though no better than the change caused by KGC), further verifying that minimizing the loss~\eqref{eqn:trace_loss} will lead to a small change in the GW distance. The runtime reported in the table confirms the analysis in Section~\ref{sec:equivalence}, showing that KGC is efficient. Moreover, the runtime of KGC(A) is nearly negligible compared with that of the baseline method used for initializing KGC(A).

\subsection{Laplacian spectrum preservation}
\label{sec:lap_spec_exp}

We evaluate the capability of the methods in preserving the Laplacian spectrum, by noting that the spectral objectives of the baseline methods differ from that of ours (see remarks after Theorem~\ref{thm: self-coarsened}). We coarsen each graph to its $10\%, 20\%, 30\%$, and $40\%$ size and evaluate the metric $\frac15 \sum_{i=1}^5 \frac{\lambda_i - \lambda_i^{(c)}}{\lambda_i}$ (the relative error of the first five largest eigenvalues). The experiment is conducted on Tumblr and the calculation of the error metric is repeated ten times.

In Figure~\ref{fig:time vs error}, we see that KGC has a comparable error to the variational methods VNGC and VEGC, while incurring a lower time cost, especially when the coarsening ratio is small. Additionally, KGC(A) consistently improves the initialization method SGC and attains the lowest error.

\begin{figure}[tbp]
    \centering
    \includegraphics[width=0.85\columnwidth]{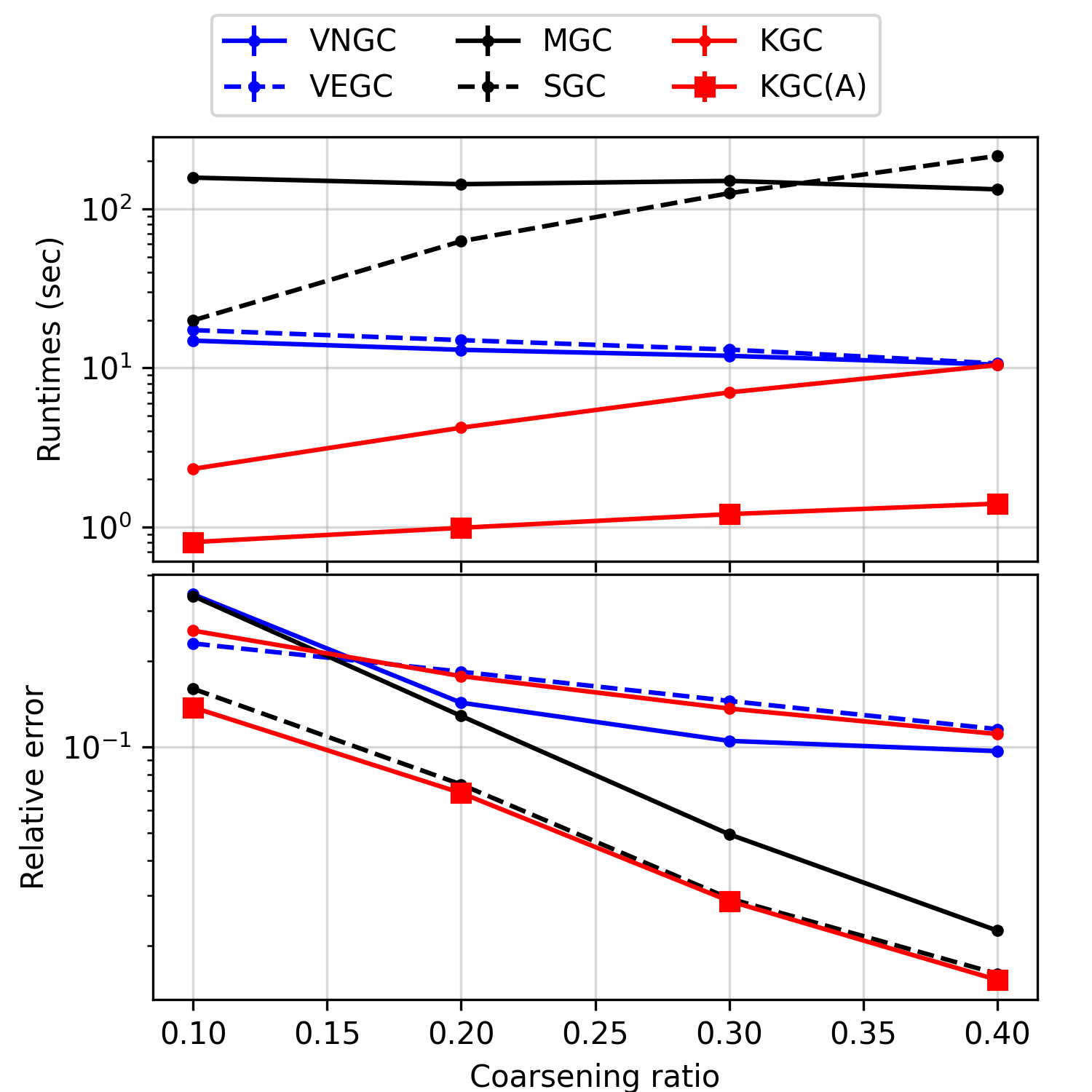}
    \caption{Runtime of coarsening methods and relative error for spectrum preservation on Tumblr. KGC(A) is initialized with SGC (the black dashed curve); its error curve (the red solid curve) is slightly below SGC.}
    \label{fig:time vs error}
\end{figure}

\begin{table*}[t]
\centering
\caption{Classification accuracy with coarsened graphs five times smaller than the original graphs.}
\label{tab:classification}
\resizebox{0.87\textwidth}{!}{
\begin{tabular}{ccccccc} 
\textbf{Datasets} & \textbf{MUTAG} & \textbf{PTC} & \textbf{PROTEINS} & \textbf{MSRC} & \textbf{IMDB} & \textbf{Tumblr} \\
\hline 
VNGC & $76.11\pm2.25$ & $56.69\pm2.52$ & $65.44\pm1.57$ & $14.92\pm1.57$  & $53.90\pm0.50$ & $50.43\pm2.62$ \\
VEGC & $84.59\pm2.02$ & $56.39\pm2.03$ & $64.08\pm1.11$& $16.80\pm2.15$  & $64.20\pm1.90$ & $48.26\pm1.71$ \\
MGC & $84.15\pm3.14$ & $54.66\pm3.59$ & $66.16\pm1.64$& $15.36\pm1.80$  & $\mathbf{69.50\pm1.42}$ & $50.14\pm2.67$ \\
SGC & $84.44\pm2.86$ & $53.79\pm2.28$ & $63.91\pm1.51$ & $16.76\pm2.50$ & $66.00\pm1.26$ & $48.53\pm2.35$ \\
\hline
KGC & $81.90\pm2.74$ & $\mathbf{61.58\pm2.49}$ & $63.45\pm0.83$ & $\mathbf{19.84\pm2.23}$ & $67.80\pm1.65$ & $52.52\pm2.81$ \\
KGC(A) & $\mathbf{86.23\pm2.69}$ & ${57.25\pm2.16}$ & $\mathbf{66.43\pm0.92}$ & $17.17\pm2.91$ & $69.20\pm1.37$ & $\mathbf{52.57\pm2.22}$ \\
\hline
EIG   & $85.61\pm1.69$ & $56.08\pm2.28$ & $64.35\pm1.43$ & $12.19\pm2.79$ & $68.70\pm1.71$ & $49.57\pm1.95$ \\
FULL   & $84.59\pm2.51$ & $54.37\pm2.12$ & $67.51\pm0.82$ & $23.58\pm2.50$ & $69.90\pm1.40$ & $52.57\pm3.36$ \\
\hline
\end{tabular}
}
\end{table*}

\subsection{Graph classification using Laplacian spectrum}
\label{sec:graph_cls}

We test the performance of the various coarsening methods for graph classification. We follow the setting of \citet{jin2020graph} and adopt the Network Laplacian Spectral Descriptor~\citep[NetLSD]{tsitsulin2018netlsd} along with a $1$-NN classifier as the classification method. We set $n=0.2N$. For evaluation, we select the best average accuracy of 10-fold cross validations among three different seeds. Additionally, we add two baselines, EIG and FULL, which respectively use the first $n$ eigenvalues and the full spectrum of $\mtx{\m L}$ in NetLSD. Similarly to before, we initialize KGC(A) with the best baseline.

Table~\ref{tab:classification} shows that preserving the spectrum does not necessarily leads to the best classification, since EIG and FULL, which use the original spectrum, are sometime outperformed by other methods. On the other hand, our proposed method, KGC or KGC(A), is almost always the best.

\subsection{Graph regression using GCNs}
\label{sec:exp regression}

We follow the setting of \citet{dwivedi2020benchmarkgnns} to perform graph regression on AQSOL and ZINC (see Appendix~\ref{app:exp_gcn} for details). For evaluation, we pre-coarsen the whole datasets, reduce the graph size by 70\%, and run GCN on the coarsened graphs. Table~\ref{tab:regression} suggests that KGC(A) initialized with VEGC (the best baseline) always returns the best test MAE. On ZINC, KGC sometimes suffers from the initialization, but its performance is still comparable to a reported MLP baseline~\citep[TestMAE $0.706\pm0.006$]{dwivedi2020benchmarkgnns}.

\begin{table}[t]
\caption{Graph regression results on AQSOL and ZINC. The asterisk symbol $*$ indicates the TestMAE improvement of KGC(A) over its VEGC initialization is statistically significant at the 95\% significance level in an paired $t$-test.}
\label{tab:regression}
\begin{subtable}[h]{\columnwidth}
\centering
\caption{AQSOL.}
\resizebox{0.95\columnwidth}{!}{
\begin{tabular}{r|ccc}
\textbf{Methods} & \textbf{TestMAE$\pm$s.d.}& \textbf{TrainMAE$\pm$s.d.} & \textbf{Epochs} \\
\hline
VNGC   & $1.403\pm0.005$ & $0.629\pm0.018$ & $135.75$ \\
VEGC  & $1.390\pm0.005$ & $0.702\pm0.003$ & $107.75$ \\
MGC   & $1.447\pm0.005$ & $0.628\pm0.012$ & $111.00$ \\
SGC   & $1.489\pm0.010$ & $0.676\pm0.021$ & $107.00$ \\
KGC & $1.389\pm0.015$ & $0.678\pm0.013$ & $112.00$  \\
KGC(A) & $\mathbf{1.383\pm0.005}^*$ & $0.657\pm0.013$ & $124.75$ \\
\hline
FULL & $1.372\pm0.020$ & $0.593\pm0.030$ & $119.50$ \\
\hline
\end{tabular}
}
\end{subtable}
\par\medskip
\hfill
\begin{subtable}[h]{\columnwidth}
\centering
\caption{ZINC.}
\resizebox{0.95\columnwidth}{!}{
\begin{tabular}{r|ccc}
\textbf{Methods} & \textbf{TestMAE$\pm$s.d.}& \textbf{TrainMAE$\pm$s.d.} & \textbf{Epochs} \\
\hline
VNGC   & $0.709\pm0.005$ & $0.432\pm0.012$ & $120.00$ \\
VEGC  & $0.646\pm0.001$ & $0.418\pm0.008$ & $138.25$ \\
MGC   & $0.677\pm0.002$ & $0.414\pm0.006$ & $112.50$ \\
SGC   & $0.649\pm0.007$ & $0.429\pm0.008$ & $111.75$ \\
KGC & $0.737\pm0.010$ & $0.495\pm0.012$ & $113.50$  \\
KGC(A) & $\mathbf{0.641\pm0.003}^*$ & $0.433\pm0.013$ & $126.50$ \\
\hline
FULL & $0.416\pm0.006$ & $0.313\pm0.011$ & $159.50$ \\
\hline
\end{tabular}
}
\end{subtable}
\end{table}

\section{Conclusions}

In this work, we propose a new perspective to study graph coarsening, by analyzing the distance of graphs on the GW space. We derive an upper bound on the change of the distance caused by coarsening, which depends on only the spectral difference $\Delta=\sum_{i=1}^n \paren{\lambda_i - \lambda^{(c)}_i}$. This bound in a way justifies the idea of preserving the spectral information as the main objective of graph coarsening, although our definition of ``spectral-preserving'' differs from prior spectral coarsening techniques. More importantly, we point out the equivalence between the bound and the objective of weighted kernel $K$-means clustering. This equivalence leads to a new coarsening method we termed KGC. Our experiment results validate the theoretical analysis, showing that KGC preserves the GW distance between graphs and improves the accuracy of graph-level classification and regression tasks.

\section*{Acknowledgements}

We wish to appreciate all the constructive and insightful comments from the anonymous reviewers.
Yun Yang's research was supported in part by U.S. NSF grant DMS-2210717.
Jie Chen acknowledges supports from the MIT-IBM Watson AI Lab.

\clearpage
\bibliography{graph_coarsening}
\bibliographystyle{icml2023}

\newpage
\appendix
\onecolumn

\section{List of notations}\label{app: list_notation}
\begin{table}[ht]
\centering
\label{tab: notation}
\begin{tabular}{r|c}
\textbf{Notations} & \textbf{Meaning}\\
\hline
$\mtx L (\mtx L^{(c)})$    & (coarsened) graph Laplacian matrix\\
$\mtx{\m L} (\mtx{\m L}^{(c)})$   & (coarsened) normalized graph Laplacian matrix\\
$G^{(c)}$   & coarsened graph\\
$\mtx D(\mtx D^{(c)})$ & (coarsened) degree matrix\\
$\mtx A(\mtx A^{(c)})$ & (coarsened) adjacency matrix\\
$\mtx C_p$ & membership matrix, coarsening matrix with entries $\in \set{0, 1}$ \\
$\mtx C_w$ & orthogonal coarsening matrix, a variant of coarsening matrix \\
$\bar{\mtx C}_w$ & weighted averaging matrix, a variant of coarsening matrix \\
$\mtx \Pi_w$ & projection matrix induced by $\mtx C_w$ \\
$\mtx W$ & diagonal node mass matrix\\
$\mtx S$($\mtx S^{(c)}$)& (coarsened) similarity matrix\\
\hline
\end{tabular}
\end{table}
We elaborate more about the family of coarsening matrices here. 
Specifically, we define the matrix $\mtx C_p\in\mb R^{n\times N}$ as
\begin{align*}
    \mtx C_p(k, i) = 
    \begin{cases}
    1 & v_i\in\m P_k\\
    0 & otherwise
    \end{cases}.
\end{align*}
Let $p_i = |\m P_i|$ be the number of nodes in the $i$-th partition $\m P_i$ and $n$ be the number of partitioning. 
Then, we define the weight matrix $\mtx W = \text{diag}(m_1,\cdots, m_N)$ and let $c_i = \sum_{j\in\m P_i}m_j$.
We can thus give the definition of the weighted averaging matrix as
\begin{align*}
    \bar{\mtx C}_w = 
    \begin{pmatrix}
        \frac{1}{c_1} & & & \\
        & \frac{1}{c_2} & &\\
        & & \cdots &\\
        & & & \frac{1}{c_n}
    \end{pmatrix}
    \mtx C_p \mtx W.
\end{align*}
and the orthogonal coarsening matrix
\begin{align*}
    \mtx C_w = 
    \begin{pmatrix}
        \sqrt{\frac{1}{c_1}} & & & \\
        & \sqrt{\frac{1}{c_2}} & &\\
        & & \cdots &\\
        & & & \sqrt{\frac{1}{c_n}}
    \end{pmatrix}
    \mtx C_p\mtx W^{\frac{1}{2}}
\end{align*}
Let the projection matrix be $\mtx \Pi_w = \mtx C_w^\intercal\mtx C_w$. We can check that $\mtx \Pi_w^2 = \mtx \Pi_w$.

\section{Derivations Omitted in the Main Text}

\subsection{Weighted graph coarsening leads to doubly-weighted Laplacian}
\label{app: doubly}

We show in the following that $\mtx C_w \mtx{\m L} \mtx C_w^\intercal$ can reproduce the normalized graph Laplacian $\mtx{\m L}^{(c)}$.
In this case, $\mtx C_w = \paren{\mtx C_p \mtx D \mtx C_p^\intercal}^{-\frac12} \mtx C_p \mtx{D}^\frac12$ and interestingly $\mtx C_p \mtx D \mtx C_p^\intercal = \mathrm{diag}\paren{\mtx C_p \mtx A \mtx C_p^\intercal \mtx 1} = \mathrm{diag}\paren{\mtx A^{(c)} \mtx 1} = \mtx D^{(c)}$, indicating
\begin{align*}
& \mtx C_w \mtx{\m L} \mtx C_w^\intercal = \mtx I_n - \mtx C_w \mtx{D}^{-\frac12} \mtx{A} \mtx{D}^{-\frac12} \mtx C_w^\intercal \\
=& \mtx I_n - \paren{\mtx D^{(c)}}^{-\frac12} \mtx C_p \mtx{D}^{\frac12} \mtx{D}^{-\frac12} \mtx{A} \mtx{D}^{-\frac12} \mtx{D}^{\frac12} \mtx C_p^\intercal \paren{\mtx D^{(c)}}^{-\frac12} \\
=& \mtx I_n - \paren{\mtx D^{(c)}}^{-\frac12} \mtx A^{(c)} \paren{\mtx D^{(c)}}^{-\frac12}
= \mtx{\m L}^{(c)}.
\end{align*}
The results above imply we can unify previous coarsening results under the weighted graph coarsening framework in this paper, with a proper choice of similarity matrix $\mtx S$ and node measure $\mu$.

\subsection{A toy example of coarsening a 3-node graph}
\label{app: toy scale}

Consider a fully-connected 3-node graph with equal weights for nodes and the partition $\set{\set{v_1}, \set{v_2, v_3}}$. 
Its similarity matrix $\mtx S = \mtx D + \mtx A$ and the three possible coarsened similarity matrices will respectively be 
\begin{align*}
\mtx S = 
\begin{pmatrix}
    2 & 1 & 1 \\
    1 & 2 & 1 \\
    1 & 1 & 2
\end{pmatrix}, \quad
\bar{\mtx{C}}_w \mtx{S} \bar{\mtx{C}}_w^{\intercal} = 
\begin{pmatrix}
    2 & 1  \\
    1 & 3/2
\end{pmatrix}, \quad
\mtx C_w \mtx{S} \mtx C_w^{\intercal} = 
\begin{pmatrix}
    2 & \sqrt{2}  \\
    \sqrt{2} & 3
\end{pmatrix}, \quad
\mtx C_p \mtx{S} \mtx C_p^{\intercal} = 
\begin{pmatrix}
    2 & 2  \\
    2 & 6
\end{pmatrix}.
\end{align*}
It is clear that $\bar{\mtx{C}}_w \mtx{S} \bar{\mtx{C}}_w^{\intercal}$, which will have the most appropriate entry magnitude to give the minimal GW-distance, is different from $\mtx C_p \mtx{S} \mtx C_p^\intercal$ proposed in \citet{jin2020graph} and \cite{DBLP:conf/iclr/CaiWW21}.
We note in $\mtx S^{(c)} = \bar{\mtx{C}}_w \mtx{S} \bar{\mtx{C}}_w^{\intercal}$ the edge weight (similarity) is still 1, since we explicitly specify the node weight of the supernode becomes $2$.
The new GW geometric framework thus decouples the node weights and similarity intensities.

\subsection{Deriving the ``Averaging'' magnitude from the constrained srGW barycenter problem}
\label{app:srgw}

We first introduce the definition of srGW divergence.
For a given graph $G$ with node mass distribution $\mtx m \in \Delta^{N-1}$ and similarity matrix $\mtx S$, we can construct another graph $G^{(c)}$ with given similarity matrix $\mtx S^{(c)}$ and \textit{unspecified} node mass distribution $\mtx m^{(c)} \in \Delta^{n-1}$. 
To better reflect the dependence of $\mathrm{GW}_2^2\paren{G, G^{(c)}}$ on node mass distribution and similarity matrix, 
we abuse the notation $\mathrm{GW}_2$ as $\mathrm{GW}_2^2\paren{\mtx S, \mtx m, \mtx S^{(c)}, \mtx m^{(c)}}$ and $\Pi\paren{\mtx m, \mtx m^{(c)}} \defeq \set{\mtx T \in \mb R_+^{N\times n}: \mtx{T1} = \mtx m, \mtx{T}^\intercal \mtx 1 
 = \mtx m^{(c)}}$ in this subsection.
\citet{vincent-cuaz2022semirelaxed} then defined
\begin{align*}
    \mathrm{srGW}_2^2\paren{\mtx S, \mtx m, \mtx S^{(c)}} \defeq \min_{\mtx m^{(c)}} \mathrm{GW}_2^2\paren{\mtx S, \mtx m, \mtx S^{(c)}, \mtx m^{(c)}},
\end{align*}
and we can further find an optimal (w.r.t.\ the srGW divergence) by solving the following srGW barycenter problem with only one input $(\mtx S, \mtx m)$, i.e.
\begin{align}
\label{eqn:srgwb}
\min_{\mtx S^{(c)}} \mathrm{srGW}_2^2\paren{\mtx S, \mtx m, \mtx S^{(c)}}.
\end{align}
We can then do the following transform to \Cref{eqn:srgwb} to reveal its connection with our proposed ``Averaging'' magnitude $\bar{\mtx{C}}_w \mtx{S} \bar{\mtx{C}}_w^{\intercal}$ for the coarsened similarity matrix $\mtx S^{(c)}$.
\begin{align*}
\min_{\mtx S^{(c)}} \mathrm{srGW}_2^2\paren{\mtx S, \mtx m, \mtx S^{(c)}} 
&= \min_{\mtx S^{(c)}} \min_{\mtx m^{(c)}} \mathrm{GW}_2^2\paren{\mtx S, \mtx m, \mtx S^{(c)}, \mtx m^{(c)}} \\
&= \min_{\mtx S^{(c)}} \min_{\mtx m^{(c)}} \min_{\mtx T \in \Pi\paren{\mtx m, \mtx m^{(c)}}} \dotp{\mtx M\paren{\mtx S, \mtx S^{(c)}, \mtx T}}{\mtx T\paren{\mtx m, \mtx m^{(c)}}} \\
&= \min_{\mtx m^{(c)}, \mtx T \in \Pi\paren{\mtx m, \mtx m^{(c)}}} \min_{\mtx S^{(c)}} \dotp{\mtx M\paren{\mtx S, \mtx S^{(c)}, \mtx T}}{\mtx T\paren{\mtx m, \mtx m^{(c)}}} \\
&= \min_{\mtx T \in \mb R_+^{N\times n}: \mtx{T1} = \mtx m} \min_{\mtx S^{(c)}} \dotp{\mtx M\paren{\mtx S, \mtx S^{(c)}, \mtx T}}{\mtx T}.
\end{align*}
In the display above, we use $\mtx M\paren{\mtx S, \mtx S^{(c)}, \mtx T}$ and $\mtx T\paren{\mtx m, \mtx m^{(c)}}$ to show the dependence terms of the cross-graph dissimilarity matrix $\mtx M$ and the transport matrix $\mtx T$.
The last equation holds since for a given transport matrix $\mtx T$ the new node mass distribution $\mtx m^{(c)}$ is uniquely decided by $\mtx m^{(c)} = \mtx T^\intercal \mtx 1$.

Notably, there is a closed-form solution for the inner minimization problem $\min_{\mtx S^{(c)}} \dotp{\mtx M\paren{\mtx S, \mtx S^{(c)}, \mtx T}}{\mtx T}$.
\citet[Equation~(14)]{peyre2016gromov} derive that the optimal $\mtx S^{(c)}$ reads
\begin{align*}
\mathrm{diag}\paren{\mtx m^{(c)}}^{-1} \mtx T^\intercal \mtx S \mtx T \mathrm{diag}\paren{\mtx m^{(c)}}^{-1}.
\end{align*}
If we then enforce the restriction that the node mass transport must be performed in a clustering manner (i.e., the transport matrix $\mtx T = \mtx W \mtx C_p^\intercal \in \mb R^{N \times n}$ for a certain membership matrix $\mtx C_p$),
we exactly have $\mtx S^{(c)} = \bar{\mtx{C}}_w \mtx{S} \bar{\mtx{C}}_w^{\intercal}$. 
The derivation above verifies the effectiveness of the ``Averaging'' magnitude we propose.

\subsection{Comparing spectral difference $\Delta$ to spectral distance in \citet{jin2020graph}}
\label{app:spectral_distance}

\citet{jin2020graph} proposed to specify the graph coarsening plan by minimizing the following \textit{full spectral distance} term (c.f. their Equation (8)):
\begin{align*}
SD_{\text{full}}\paren{G, G^{(c)}} &= \sum_{i=1}^{k_1} \left|\boldsymbol{\lambda}(i)-\boldsymbol{\lambda}_c(i)\right|+\sum_{i=k_1+1}^{k_2}|\boldsymbol{\lambda}(i)-1| + \sum_{i=k_2+1}^N\left|\boldsymbol{\lambda}(i)-\boldsymbol{\lambda}_c(i-N+n)\right| \\
&= \sum_{i=1}^{k_1} \paren{\boldsymbol{\lambda}_c(i)-\boldsymbol{\lambda}(i)} +  \sum_{i=k_2+1}^N \paren{\boldsymbol{\lambda}(i)-\boldsymbol{\lambda}_c(i-N+n)}
+ \sum_{i=k_1+1}^{k_2}|\boldsymbol{\lambda}(i)-1|,
\end{align*}
where $\boldsymbol{\lambda}(i)$'s and $\boldsymbol{\lambda}_c(i)$'s correspond to the eigenvalues (in an ascending order) of the normalized Laplacian $\mtx{\m L}$ and $\mtx{\m L}^{(c)}$,
and $k_1$ and $k_2$ are defined as $k_1 = \argmax_i \set{i: \boldsymbol{\lambda}_c(i) < 1}, k_2 = N - n + k_1$.
The last equation holds due to their Interlacing Property~4.1 (similar to \Cref{thm: poincare_sep}).

We note (\rom 1) the two ``spectral'' loss functions, $\Delta$ and $SD_{\text{full}}$, refer to different spectra. We leverage the spectrum of $C_w U C_w^T$ while they focus on graph Laplacians. Our framework is more general and takes node weights into consideration.
(\rom 2) They actually divided $\boldsymbol{\lambda}_c(i)$’s into two sets and respectively compared them to $\boldsymbol{\lambda}(i)$ and $\boldsymbol{\lambda}(i+N-n)$; 
the signs for $\lambda(i) - \lambda_c(i)$ and $\lambda(i+N-n) - \lambda_c(i)$ are thus different.

\subsection{Time complexity of Algorithm~\ref{alg:WKKmeans}}
\label{sec:alg time}

We first recall the complexity of Algorithm~\ref{alg:WKKmeans} stated in Section~\ref{sec:equivalence}.
Let $T$ be the upper bound of the $K$-means iteration, the time complexity of Algorithm~\ref{alg:WKKmeans} is $\m O\paren{T \paren{M+Nn}}$, 
where $M = \mathrm{nnz}(\mtx S)$ is the number of non-zero elements in $\mtx S$. 

The cost mainly comes from the computation of the ``distance'' (\ref{eqn:kmeans_dist}), which takes $\m O(M)$ time to obtain the second term in \Cref{eqn:kmeans_dist} and pre-compute the third term (independent of $i$);
obtaining the exact $\mathrm{dist}_j(i)$ for all the $Nn$ $(i, j)$ pairs requires another $\m O(Nn)$ time.
Compared to the previous spectral graph coarsening method~\citep{jin2020graph}, Algorithm~\ref{alg:WKKmeans} removes the partial sparse eigenvalue decomposition, which takes $\m O\paren{R(M n + N n^2)}$ time using Lanczos iteration with $R$ restarts.
The $K$-means part of spectral graph coarsening takes $\m O(TNn^2)$ for a certain choice of the eigenvector feature dimension $k_1$ \citep[Section~5.2]{jin2020graph},
while weighted kernel $K$-means clustering nested in our algorithm can better utilize the sparsity in the similarity matrix.

\section{Details of Experiments}
\label{app: exp}

We first introduce the hardware and the codebases we utilize in the experiments.
The algorithms tested are all implemented in unoptimized Python code, and run with one core of a server CPU (Intel(R) Xeon(R) Gold 6240R CPU @ 2.40GHz) on Ubuntu 18.04.
Specifically, our method KGC is developed based on the (unweighted) kernel $K$-means clustering program provided by \citet{ivashkin2016logarithmic};
for the other baseline methods, the variation methods VNGC and VEGC are implemented by \citet{loukas2019graph}, and the spectrum-preserving methods MGC and SGC are implemented by \citet{jin2020graph};
The S-GWL method \citep{xu2019scalable, chowdhury2021generalized} is tested as well (can be found in the GitHub repository of this work), while this algorithm cannot guarantee that the number of output partition is the same as specified.
For the codebases, we implement the experiments mainly using the code by \citet{jin2020graph, dwivedi2020benchmarkgnns}, for graph classification and graph regression respectively.
More omitted experimental settings in Section~\ref{sec:exp} will be introduced along the following subsections.

\begin{table}[tbp]
\caption{Average squared GW distance $GW_2^2(G^{(c)}, G)$ on PTC and IMDB dataset.} 
\label{tab:avg_gw}
\begin{center}
\begin{tabular}{c|ccccc}
\textbf{Dataset} &\textbf{Methods} &\textbf{$c=0.3$}$\downarrow$ &\textbf{$c=0.5$}$\downarrow$ &\textbf{$c=0.7$}$\downarrow$ &\textbf{$c=0.9$}$\downarrow$ \\
\hline
\multirow{6}{*}{PTC} &VNGC    & 0.05558 & 0.04880 & 0.03781 & 0.03326 \\
&VEGC    & 0.03064 & 0.02352 & 0.01614 & 0.00927 \\
&MGC     & 0.05290 & 0.04360 & 0.02635 & 0.00598 \\
&SGC     & 0.03886 & 0.03396 & 0.02309 & 0.00584 \\
&KGC     & 0.03332 & 0.02369 & \textbf{0.01255} & \textbf{0.00282} \\
&KGC(A)  & \textbf{0.03055} & \textbf{0.02346} & 0.01609 & 0.00392 \\
\hline
\multirow{6}{*}{IMDB} &VNGC   & 0.05139 & 0.05059 & 0.05043 & 0.05042 \\
&VEGC   & 0.02791 & \textbf{0.02106} & 0.01170 & 0.00339 \\
&MGC    & \textbf{0.02748} & 0.02116 & 0.01175 & 0.00339 \\
&SGC    & 0.02907 & 0.02200 & 0.01212 & 0.00352 \\
&KGC    & 0.02873 & 0.02111 & \textbf{0.01137} & \textbf{0.00320} \\
&KGC(A) & \textbf{0.02748} & \textbf{0.02106} & 0.01170 & 0.00337 \\
\hline
\end{tabular}

\end{center}
\end{table}

\subsection{Details of GW distance approximation in Section~\ref{sec:exp_distMtx}}
\label{app:gw_distMtx}

\begin{table}[tbp]
\caption{Average empirical bound gaps (in \Cref{thm: self-coarsened}) on PTC and IMDB dataset.} 
\label{tab:avg_gap}
\begin{center}
\begin{tabular}{c|ccccc}
\textbf{Dataset} &\textbf{Methods} &\textbf{$c=0.3$} &\textbf{$c=0.5$} &\textbf{$c=0.7$} &\textbf{$c=0.9$} \\
\hline
\multirow{6}{*}{PTC} &VNGC   & 0.06701 & 0.06671 & 0.05393 & 0.04669 \\
&VEGC   & 0.06246 & 0.06129 & 0.04424 & 0.02577 \\
&MGC    & 0.03203 & 0.03200 & 0.02167 & 0.00540 \\
&SGC    & 0.04599 & 0.04156 & 0.02488 & 0.00554 \\
&KGC    & 0.05145 & 0.05173 & 0.03530 & 0.00852 \\
&KGC(A) & 0.06519 & 0.06402 & 0.04702 & 0.00372 \\
\hline
\multirow{6}{*}{IMDB} &VNGC   & 0.009281 & 0.00927  & 0.009278 & 0.009268 \\
&VEGC   & 0.016879 & 0.016735 & 0.01636  & 0.008221 \\
&MGC    & 0.017307 & 0.01663  & 0.016309 & 0.008179 \\
&SGC    & 0.015719 & 0.015793 & 0.015934 & 0.008049 \\
&KGC    & 0.016054 & 0.016679 & 0.016687 & 0.008347 \\
&KGC(A) & 0.017307 & 0.016735 & 0.01636  & 0.008177 \\
\hline
\end{tabular}

\end{center}
\end{table}

We compute the pair-wise $\GW_2$ distance matrix $\mtx G$ for graphs in PTC and IMDB, with the normalized signless Laplacians set as the similarity matrices. 
For the computation of the $\GW_2$ distance, we mainly turn to the popular OT package \texttt{POT} \citep[Python Optimal Transport]{flamary2021pot}.

The omitted tables of average squared GW distance $GW_2^2(G^{(c)}, G)$ and average empirical bound gaps are presented in Tables~\ref{tab:avg_gw} and \ref{tab:avg_gap}.

Regarding time efficiency, we additionally remark our method KGC works on the dense graph matrix, even though it has a better theoretical complexity using a sparse matrix; 
for small graphs in MUTAG, directly representing them by dense matrices would even be faster, when a modern CPU is used. 

\subsection{Details of Laplacian spectrum preservation in Section~\ref{sec:lap_spec_exp}}

We mainly specify the evaluation metric for spectrum preservation in this subsection. 
Following the eigenvalue notation in Theorem~\ref{thm: self-coarsened}, we define the top-$5$ eigenvalue relative error as $\frac15 \sum_{i=1}^5 \frac{\lambda_i - \lambda_i^{(c)}}{\lambda_i}$.
Here for all coarsening methods $\lambda_i - \lambda_i^{(c)}$ is always non-negative thanks to Poincar\'e separation theorem (Theorem~\ref{thm: poincare_sep}).

\subsection{Details of Graph classification with Laplacian spectrum in Section~\ref{sec:graph_cls}}

We remark the graph classification experiments are mainly adapted from the similar experiments in \citet[Section~6.1]{jin2020graph}.
For MUTAG and PTC datasets, we apply Algorithm~\ref{alg:WKKmeans} to $\mtx D + \mtx A$; for the other four datasets, we utilize the normalized signless Laplacian $\mtx{I}_N + \mtx{D}^{-\frac12} \mtx{A} \mtx{D}^{-\frac12}$.

\subsection{Details of Graph regression with GCNs in Section~\ref{sec:exp regression}}
\label{app:exp_gcn}

We mainly follow the GCN settings in \citet{dwivedi2020benchmarkgnns}.
More detailed settings are stated as follows.

\textbf{Graph regression.} The procedure of graph regression is as follows:
Taking a graph $G$ as input, a GCN will return a graph embedding $\mtx y_{G} \in \mb R^{d}$; 
\citet{dwivedi2020benchmarkgnns} then pass $\mtx y_{G}$ to an MLP and compute the prediction score $y_{\text{pred}} \in \mathbb R$ as
$$
y_{\text{pred}} = \mtx P \cdot \text{ReLU}(\mtx Q \mtx y_{\mathcal G}),
$$
where $\mtx P \in \mathbb R^{1 \times d}, \mtx Q \in \mathbb R^{d \times d}$. 
They will then use $|y_{\text{pred}} - y|$, the $L_1$ loss (the MAE metric in \Cref{tab:regression}) between the predicted score $y_{\text{pred}}$ and the groundtruth score $y$, both in training and performance evaluation.

\textbf{Data splitting.} They apply a scaffold splitting~\citep{hu2020open} to the AQSOL dataset in the ratio $8 : 1 : 1$ to have $7831, 996$, and $996$ samples for train, validation, and test sets.

\textbf{Training hyperparameters.} 
For the learning rate strategy across all GCN models, we follow the existing setting to choose the initial learning rate as $1 \times 10^{-3}$, the reduce factor is set as $0.5$, and the stopping learning rate is $1 \times 10^{-5}$.
Also, all the GCN models tested in our experiments share the same architecture---the network has $4$ layers and $108442$ tunable parameters.

As for the concrete usage of graph coarsening to GCNs, we discuss the main idea in \Cref{sec:gcgcn} and leave the technical details to \Cref{app: gcgcn}.

\subsubsection{Application of graph coarsening to GCNs}
\label{sec:gcgcn}

Motivated by the GW setting, we discuss the application of graph coarsening to a prevailing and fundamental graph model---GCN\footnote{We omit the introduction to GCN here and refer readers to \citet{kipf2016semi} for more details.}.
After removing the bias term for simplicity and adapting the notations in this paper, we take a vanilla $1$-layer GCN as an example and formulate it as
\begin{align*}
    \mtx y^\intercal = \frac{\mtx 1_N^\intercal}{N} \sigma\paren{\mtx D^{-\frac12}\mtx A \mtx D^{-\frac12} \mtx H \mtx W_{\mathrm{GCN}}}, 
\end{align*}
where $\mtx y$ is the graph representation of a size-$N$ graph associated with the adjacency matrix $\mtx A$, $\sigma$ is an arbitrary activation function, $\mtx H$ is the embedding matrix for nodes in the graph, and $\mtx W_{\mathrm{GCN}}$ is the weight matrix for the linear transform in the layer.
We take the mean operation $\frac{\mtx 1_N^\intercal}{N}$ in the GCN as an implication of even node weights (therefore no $w$ subscript for coarsening matrices), and intuitively incorporate graph coarsening into the computation by solely replacing $\mtx D^{-\frac12}\mtx A \mtx D^{-\frac12}$ with $\mtx \Pi \mtx D^{-\frac12}\mtx A \mtx D^{-\frac12} \mtx \Pi$ and denote $\mtx y^{(c)}$ as the corresponding representation.
We have
\begin{align}
\paren{\mtx y^{(c)}}^\intercal = \mtx c^\intercal \sigma\paren{{\bar{\mtx C} \mtx D^{-\frac12}\mtx A \mtx D^{-\frac12} \mtx C_p} {\bar{\mtx C} \mtx H \mtx W_{\mathrm{GCN}}}}, 
\label{eqn:gcgcn}
\end{align}
and the graph matrix $\mtx D^{-\frac12}\mtx A \mtx D^{-\frac12}$ is supposed to be coarsened as $\bar{\mtx C} \mtx D^{-\frac12}\mtx A \mtx D^{-\frac12} \mtx C_p$, which can be extended to multiple layers.
Due to the space limit, we defer the derivation to Appendix~\ref{app: gcgcn} and solely leave some remarks here.

The propagation above is guided by the GW setting along this paper: even the nodes in the original graph are equally-weighted, the supernodes after coarsening can have different weights, which induces the new readout operation $\mtx c^\intercal$ ($\mtx c$ is the mass vector of the supernodes).
Furthermore, an obvious difference between \Cref{eqn:gcgcn} and previous designs~\citep{huang2021scaling, DBLP:conf/iclr/CaiWW21} of applying graph coarsening to GNNs is the asymmetry of the coarsened graph matrix $\bar{\mtx C} \mtx D^{-\frac12}\mtx A \mtx D^{-\frac12} \mtx C_p$.
The design (\ref{eqn:gcgcn}) is tested for the graph regression experiment in Section~\ref{sec:exp regression}.
More technical details are introduced in the subsequent subsection.

\subsubsection{Derivation of the GCN propagation in the GW setting}
\label{app: gcgcn}

We consider a general layer-$L$ GCN used by \citet{dwivedi2020benchmarkgnns}. We first recall the definition (normalization modules are omitted for simplicity):
\begin{align*}
    \mtx y^\intercal &= \frac{\mtx 1_N^\intercal}{N} \mtx H^{(L)}, \\
    \mtx H^{(l)} &= \sigma (\mtx Z^{(l)}), 
    \quad \mtx Z^{(l)} = \paren{\mtx D^{-\frac12}\mtx A \mtx D^{-\frac12}} \mtx H^{(l-1)} \mtx W^{(l)}, \quad \forall l \in [L], \\
    \mtx H^{(0)} &\defeq \mtx X, \quad \text{the embedding matrix for each nodes},
\end{align*}
where $\sigma$ is an activation function and from now on we will abuse $\mtx P$ here to denote $\mtx D^{-\frac12}\mtx A \mtx D^{-\frac12}$;
$\mtx H^{(l)}$ is the embedding matrix of the graph nodes in the $l$-th layer, 
and $\mtx W^{(l)}$ is the weight matrix of the same layer. 

To apply the coarsened graph, we enforce the following regulations:
\begin{align*}
    \mtx H^{(c, l)} &= \sigma (\mtx Z^{(c, l)}), 
    \quad \mtx Z^{(c, l)} = \paren{\mtx \Pi \mtx P \mtx \Pi} \mtx H^{(c, l-1)} \mtx W^{(l)}, \quad \forall l \in [L].
\end{align*}
To reduce the computation in node aggregation, we utilize the two properties that $\mtx \Pi = \mtx C_p^\intercal \bar{\mtx C}$ and $\sigma\paren{\mtx C_p^\intercal \mtx B} = \mtx C_p^\intercal \sigma\paren{\mtx B}$, for any element-wise activation function and matrix $\mtx B$ with a proper shape (considering $\mtx C_p^\intercal$ simply ``copy'' the rows from $\mtx B$);
for the top two layers, we then have
\begin{align*}
    \mtx y^\intercal &= \frac{\mtx 1_N^\intercal}{N} \sigma (\mtx C_p^\intercal \bar{\mtx C} \mtx P \mtx \Pi \mtx H^{(c, L-1)} \mtx W^{(L)})= 
    \frac{\mtx 1_N^\intercal}{N} \mtx C_p^\intercal \sigma\paren{\bar{\mtx C} \mtx P \mtx C_p^\intercal \bar{\mtx C} \mtx H^{(c, L-1)} \mtx W^{(L)}}, \\ 
    \mtx H^{(c, L-1)} &= \sigma\paren{\mtx C_p^\intercal \bar{\mtx C} \mtx P \mtx C_p^\intercal \bar{\mtx C} \mtx H^{(c, L-2)} \mtx W^{(L-1)}} = \mtx C_p^\intercal \sigma\paren{\bar{\mtx C} \mtx P \mtx C_p^\intercal \bar{\mtx C} \mtx H^{(c, L-2)} \mtx W^{(L-1)}}.
\end{align*}
Note $\mtx C_p^\intercal \bar{\mtx C} \mtx C_p^\intercal = \mtx C_p^\intercal$; for the top two layers we finally obtain
\begin{align*}
    \mtx y^\intercal = \frac{\mtx 1_N^\intercal}{N} \mtx C_p^\intercal \sigma\paren{\bar{\mtx C} \mtx P \mtx C_p^\intercal \bar{\mtx C} \mtx H^{(c, L-1)} \mtx W^{(L)}} = 
    \mtx c^\intercal \sigma\paren{\bar{\mtx C} \mtx P \mtx C_p^\intercal \sigma\paren{\bar{\mtx C} \mtx P \mtx C_p^\intercal \bar{\mtx C} \mtx H^{(c, L-2)} \mtx W^{(L-1)}} \mtx W^{(L)}},
\end{align*}
implying that we can replace $\mtx P$ with $\bar{\mtx C} \mtx P \mtx C_p^\intercal$ in the propagation of the top two layers.
The trick can indeed be repeated for each layer, and specifically, in the bottom layer we can have
\begin{align*}
    \mtx H^{(c, 1)} = 
    \mtx C_p^\intercal \sigma\brkt{\paren{\bar{\mtx C} \mtx P \mtx C_p^\intercal} \paren{\bar{\mtx C} \mtx H^{(0)}} \mtx W^{(1)}},
\end{align*}
which well corresponds to the node weight concept in the GW setting: $\bar{\mtx C} \mtx H^{(0)}$ uses the average embedding of the nodes within the cluster to represent the coarsened centroid.
We then finish the justification of the new GCN propagation in \Cref{eqn:gcgcn}.

\section{Useful Facts}
\label{sec:facts}

\subsection{Poincar\'e separation theorem}
\label{sec:interlacing}

For convenience of the reader, we repeat Poincar\'e separation theorem~\citep{bellman1997introduction} in this subsection.

\begin{theorem}[Poincar\'e separation theorem]\label{thm: poincare_sep}
Let $\mtx A$ be an $N \times N$ real symmetric matrix and $\mtx C$ an $n \times N$ semi-orthogonal matrix such that $\mtx C \mtx C^\intercal = \mtx I_n$. 
Denote by $\lambda_{i}, i = 1, 2, \dots, N$ and $\lambda^{(c)}_{i}, i = 1, 2, \dots, n$ the eigenvalues of $\mtx A$ and $\mtx C \mtx A \mtx C^\intercal$, respectively (in descending order). We have
\begin{align}
\lambda_{i} \geq \lambda^{(c)}_{i} \geq \lambda_{{N-n+i}},
\end{align}
\end{theorem}

\subsection{Ruhe's trace inequality}
\label{sec:von_neumann}

For convenience of the reader, Ruhe's trace inequality~\citep{ruhe1970perturbation} is stated as follows.

\begin{lemma}[Ruhe's trace inequality]\label{lem: Ruhe_ineq}
If $\mtx A, \mtx B$ are both $N \times N$ PSD matrices with eigenvalues,
\begin{align*}
\lambda_1 \geq \cdots \geq \lambda_N \geq 0, \quad
\nu_1 \geq \cdots \geq \nu_N \geq 0,
\end{align*}
then
\begin{align*}
\sum_{i=1}^N \lambda_i \nu_{N-i+1} \leq \mathrm{tr}(AB) \leq \sum_{i=1}^N \lambda_i \nu_i.
\end{align*}
\end{lemma}

\section{Proof of Theorem~\ref{thm: self-coarsened} and Theorem~\ref{thm: coarsen12}}\label{ref:gw2cg}
We will first prove an intermediate result Lemma~\ref{thm: coarsen1} for coarsened graph $G_1^{(c)}$ and un-coarsend graph $G_2$ to introduce necessary technical tools for Theorem~\ref{thm: self-coarsened} and Theorem~\ref{thm: coarsen12}.
The ultimate proof of Theorem~\ref{thm: self-coarsened} and Theorem~\ref{thm: coarsen12} will be stated in Appendix~\ref{app: self-coarsened} and Appendix~\ref{app: coarsen12} respectively.

\subsection{Lemma \ref{thm: coarsen1} and Its Proof}

We first give the complete statement of \Cref{thm: coarsen1}.
In Lemma~\ref{thm: coarsen1}, we consider the case in which only graph $G_1$ is coarsened, and the notations are slightly simplified: 
when the context is clear, the coarsening-related terms without subscripts specific to graphs, e.g.\ $\mtx C_{p}, \bar{\mtx C}_{w}$, are by default associated with $G_1$ unless otherwise announced.
We follow the simplified notation in the statement and proof of Theorem~\ref{thm: self-coarsened}, which focuses on solely $\text{GW}_2(G^{(c)}, G)$ and the indices $1, 2$ are not involved.
For Theorem~\ref{thm: coarsen12}, we will explicitly use specific subscripts, such as $\mtx C_{p, 1}, \bar{\mtx C}_{w, 2}$, for disambiguation.
\begin{lemma}\label{thm: coarsen1}
Let $\mtx P = \mtx W_1^{-\frac12} \mtx T^* \mtx W_2^{-\frac12}$. If both $\mtx S_1$ and $\mtx S_2$ are PSD, we have
\begin{equation}\label{thm:one-side-bd}
\begin{aligned}
\lvert\GW_2^2(G_1, G_2) - \GW_2^2(G_1^{(c)}, G_2)\rvert 
    &\leq \max \left\{ \vphantom{\paren{\sum_{i=1}^{n_1}}} \lambda_{N_1-n_1+1} \sum_{i=1}^{n_1} \paren{\lambda_i - \lambda^{(c)}_i} + C_{\mtx U, n_1}, \right. \\
    & \qquad \qquad \left. 2\paren{\nu_{N_1-n_1+1} \sum_{i=1}^{n_1} \paren{\lambda_i - \lambda^{(c)}_i} + C_{\mtx U, \mtx V, n_1}} \right\}.
\end{aligned}
\end{equation}
Here, $\lambda_1^{(c)}\geq \cdots \geq \lambda_{n}^{(c)}$ are eigenvalues of $\mtx\Pi_w\mtx U\mtx \Pi_w$,
$\mtx U = \mtx W_1^\frac12 \mtx{S}_1 \mtx W_1^\frac12$ with eigenvalues $\lambda_{1} \geq \lambda_{2} \geq \cdots \geq \lambda_{N_1}$, and $\mtx V = \mtx P \mtx W_2^{\frac12} \mtx S_2 \mtx W_2^{\frac12} \mtx P^\intercal$ with eigenvalues $\nu_{1} \geq \nu_{2}\geq \cdots \geq \nu_{N_1}$, and we let $C_{\mtx U, n_1} \defeq \sum_{i=1}^{n_1} \lambda_{i} (\lambda_{i} - \lambda_{N_1-n_1+i}) + \sum_{i=n_1+1}^{N_1} \lambda_{i}^2$ and $C_{\mtx U, \mtx V, n_1} \defeq \sum_{i=1}^{n_1} \lambda_{i} \paren{\nu_{i} - \nu_{N_1-i+1}} + \sum_{i=n_1+1}^{N_1} \lambda_{i} \nu_{i}$ be two non-negative constants.
\end{lemma}

\begin{remark}
\normalfont
Take $G_2=G_1$, and we have $\mtx T^\ast = \mtx W_1$ which implies $\mtx P = \mtx I_N$ and $\mtx V = \mtx W_1^{\frac{1}{2}}\mtx S_1\mtx W_1^{\frac{1}{2}} = \mtx U$. 
This directly leads to the bound in Theorem~\ref{thm: self-coarsened} though with an additional factor $2$. 
In Appendix~\ref{app: self-coarsened} we will show the approach to obtain a slightly tighter bound removing the unnecessary factor $2$.
\end{remark}

To illustrate our idea more clearly, we will start from the following decomposition of $\GW_2$ distance. The detailed proofs of all the lemmas in this section will be provided to Appendix \ref{app: proof_in_main_thm} for the readers interested. 
\begin{lemma}\label{lem: decomp_GW}
For any two graphs $G_1$ and $G_2$, we have
\begin{align*}
    \GW_2^2(G_1, G_2) =  I_1 + I_2 - 2I_3,
\end{align*}
where
\begin{align*}
I_1 = \Tr \paren{\paren{\mtx W_1^\frac12 \mtx{S_1} \mtx W_1^\frac12} \paren{\mtx W_1^\frac12 \mtx{S_1} \mtx W_1^\frac12}}, \qquad
I_2 = \Tr \paren{\paren{\mtx W_2^\frac12 \mtx{S_2} \mtx W_2^\frac12} \paren{\mtx W_2^\frac12 \mtx{S_2} \mtx W_2^\frac12}}
\end{align*}
do not depend on the choice of transport map, while
\begin{align*}
    I_3 = \Tr \paren{\mtx S_1 \mtx T^* \mtx S_2 \paren{\mtx T^*}^{\intercal}}
\end{align*}
requires the optimal transport map $\mtx T^\ast$ from the graph $G_1$ to the graph $G_2$.
\end{lemma}

Replacing the $G_1$-related terms in \Cref{eqn:tr_gw} with their $G_1^{(c)}$ counterparts, we know that the distance between $G_1^{(c)}$ and $G_2$ is $\text{GW}_2(G_1^{(c)}, G_2) = I_1' + I_2 - 2I_3'$, where:
\begin{align}
\label{eqn:tr_gwc}
I_1' = \Tr \paren{\brkt{\paren{\mtx W_1^{(c)}}^\frac12 \mtx{S_1^{(c)}} \paren{\mtx W_1^{(c)}}^\frac12}^2}, \qquad
I_3' = \Tr \paren{\mtx S_1^{(c)} \mtx T_{co}^* \mtx S_2 \paren{\mtx T_{co}^*}^{\intercal}}.
\end{align}
$\mtx T_{co}^\ast \in \Pi(\mu_1^{(c)}, \mu_2)$ (\texttt{co} represents the transport from the ``c''oarsened graph to the ``o''riginal graph) here is the optimal transport matrix induced by the GW distance and $\mtx{S}_1^{(c)} \defeq \bar{\mtx C}_w \mtx{S}_1 \bar{\mtx C}_w^\intercal$.

The key step to preserve the GW distance is to control the difference $|I_1 - I_1'|$ and $|I_3 - I_3'|$, since $I_2 = I'_2$ will cancel each other. We will start from the bound of $|I_1 - I_1'|$.

\begin{lemma}\label{lem: bd_diff_I1}
Let $\mtx U \defeq \mtx W_1^\frac12 \mtx{S}_1 \mtx W_1^\frac12$. If the similarity matrix $\mtx S_1\in\mb R^{N\times N}$ is PSD, we have 
\begin{align*}
    0\leq I_1 - I_1' \leq \lambda_{N-n+1} \sum_{i=1}^n \paren{\lambda_{i} - \lambda^{(c)}_{i}} + C_{\mtx U, n}.
\end{align*}
Here, $\lambda_{1} \geq \lambda_{2} \geq \cdots \geq \lambda_{N}$ are the eigenvalues of $\mtx U$, and $C_{\mtx U, n} \defeq \sum_{i=1}^n \lambda_{i} (\lambda_{i} - \lambda_{N-n+i}) + \sum_{i=n+1}^N \lambda_{i}^2$ is non-negative.
\end{lemma}

We can similarly bound $I_3 - I'_3$ with an additional tool Ruhe's trace inequality (a variant of Von Neumann's trace inequality specific to PSD matrices, c.f.\ Appendix~\ref{sec:von_neumann}). 
\begin{lemma}\label{lem: bd_diff_I3_further}
Let $\mtx P = \mtx W_1^{-\frac12} \mtx T^* \mtx W_2^{-\frac12}$ be the normalized optimal transport matrix, and $\mtx V \defeq \mtx P \mtx W_2^{\frac12} \mtx S_2 \mtx W_2^{\frac12} \mtx P^\intercal$ with eigenvalues $\nu_{1} \geq \nu_{2}\geq \cdots \geq \nu_{N}$. If both $\mtx S_1$ and $\mtx S_2$ are PSD, we have
\begin{align*}
    0\leq I_3 - I_3' \leq \nu_{N-n+1} \sum_{i=1}^n \paren{\lambda_i - \lambda^{(c)}_i} + C_{\mtx U, \mtx V, n}.
\end{align*}
Here, $C_{\mtx U, \mtx V, n} \defeq \sum_{i=1}^n \lambda_{i} \paren{\nu_{i} - \nu_{N-i+1}} + \sum_{i=n+1}^N \lambda_{i} \nu_{i}$ is non-negative.
\end{lemma}
Now we have
\begin{align*}
    \lvert\GW_2^2(G_1, G_2) - \GW_2^2(G_1^{(c)}, G_2)\rvert = \lvert I_1 - I_1' + 2(I_3' - I_3)\rvert \leq \max\{I_1 - I_1', 2(I_3 - I_3')\}
\end{align*}
considering that $I_1 - I'_1 \geq 0$ and $I'_3 - I_3 \leq 0$. 
Then, combining all the pieces above yields the bound in Lemma~\ref{thm: coarsen1}.

\subsection{Proof of Theorem \ref{thm: self-coarsened}}
\label{app: self-coarsened}

With the above dissection of the terms $I'_1, I'_3$, we can now give a finer control of the distance $\text{GW}_2^2(G_1^{(c)}, G_1)$.
We first expand $\text{GW}_2^2(G_1^{(c)}, G_1)$ as
\begin{align*}
\text{GW}_2^2(G_1^{(c)}, G_1) = I_1 + I'_1 - 2 \Tr \paren{\mtx S_1^{(c)} \mtx T_{co}^* \mtx S_1 \paren{\mtx T_{co}^*}^{\intercal}},
\end{align*}
where the notation $\mtx T_{co}^*$ is now abused as the optimal transport matrix induced by $\mathrm{GW}_2(G_1^{(c)}, G_1)$.
Applying the optimality inequality in Lemma~\ref{lem: bd_diff_I3}, we have
\begin{align*}
\text{GW}_2^2(G_1^{(c)}, G_1) \leq I'_1 + I_1 - 2 \Tr \paren{\mtx S_1^{(c)} \mtx C_{p} \mtx W_1 \mtx S_1 \mtx W_1 \mtx C_{p}^\intercal},
\end{align*}
where we remark $\mtx C_{p} \mtx W_1$ is a qualified transport matrix for $G_1^{(c)}$ and $G_1$.

To further simplify the upper bound above, we show the equivalence between $I'_1$ and $\Tr \paren{\mtx S_1^{(c)} \mtx C_{p} \mtx W_1 \mtx S_1 \mtx W_1 \mtx C_{p}^\intercal}$ as follows:
\begin{align*}
\Tr \paren{\mtx S_1^{(c)} \mtx C_{p} \mtx W_1 \mtx S_1 \mtx W_1 \mtx C_{p}^\intercal}
&= \Tr \paren{\bar{\mtx C}_{w} \mtx S_1 \bar{\mtx C}_{w} \mtx C_{p} \mtx W_1 \mtx S_1 \mtx W_1 \mtx C_{p}^\intercal} = \Tr \paren{\mtx \Pi_{w} \mtx U \mtx \Pi_{w} \mtx U}
= \Tr \paren{\mtx \Pi_{w} \mtx \Pi_{w} \mtx U \mtx \Pi_{w} \mtx \Pi_{w} \mtx U} \\
&= \Tr \paren{\mtx \Pi_{w} \mtx U \mtx \Pi_{w} \mtx \Pi_{w} \mtx U \mtx \Pi_{w} } = I'_1.
\end{align*}
Combing the above pieces together, we obtain
\begin{align*}
\text{GW}_2^2(G_1^{(c)}, G_1) \leq I_1 + I'_1 - 2 I'_1 \leq \lambda_{N-n+1} \sum_{i=1}^n \paren{\lambda_{i} - \lambda^{(c)}_{i}} + C_{\mtx U, n},
\end{align*}
which uses Lemma~\ref{lem: bd_diff_I1} for the last inequality.
The proof of Theorem \ref{thm: self-coarsened} is now complete.

\begin{remark}
\normalfont
Theorem~\ref{thm: self-coarsened} can be leveraged to directly control $\abs{\text{GW}_2(G_1^{(c)}, G_2^{(c)}) - \text{GW}_2(G_1, G_2)}$.
Note that $\text{GW}_2$ is a pseudo-metric and satisfies triangular inequality \citep{chowdhury2019gromov}, which implies
\begin{align*}
    \text{GW}_2(G_1^{(c)}, G_2^{(c)}) - \text{GW}_2(G_1, G_2) 
    &\leq \text{GW}_2(G_1^{(c)}, G_1) + \text{GW}_2(G_1, G_2) + \text{GW}_2(G_2, G_2^{(c)}) - \text{GW}_2(G_1, G_2)\\
    &= \text{GW}_2(G_1^{(c)}, G_1) + \text{GW}_2(G_2, G_2^{(c)}) \\
    &\leq \sqrt{2} \paren{\text{GW}_2^2(G_1^{(c)}, G_1) + \text{GW}_2^2(G_2, G_2^{(c)})}^\frac12,
\end{align*}
and similarly we have
\begin{align*}
    \text{GW}_2(G_1, G_2) - \text{GW}_2(G_1^{(c)}, G_2^{(c)}) \leq \text{GW}_2(G_1^{(c)}, G_1) + \text{GW}_2(G_2, G_2^{(c)}) 
\leq \sqrt{2} \paren{\text{GW}_2^2(G_1^{(c)}, G_1) + \text{GW}_2^2(G_2, G_2^{(c)})}^\frac12.
\end{align*}
This implies 
\begin{align*}
    &\quad\big|\text{GW}_2(G_1^{(c)}, G_2^{(c)}) - \text{GW}_2(G_1, G_2)\big| \leq \sqrt{2} \paren{\text{GW}_2^2(G_1^{(c)}, G_1) + \text{GW}_2^2(G_2, G_2^{(c)})}^\frac12 \\
    &\leq \sqrt{2} \bigg(\lambda_{1, N_1-n_1+1} \sum_{i=1}^{n_1} \paren{\lambda_{1,i} - \lambda^{(c)}_{1,i}} + C_{\mtx U_1, n_1} + 
    \lambda_{2, N_2-n_2+1} \sum_{i=1}^{n_2} \paren{\lambda_{2,i} - \lambda^{(c)}_{2,i}} + C_{\mtx U_2, n_2}\bigg)^\frac12.
\end{align*}
Here, the last inequality is due to Theorem~\ref{thm: self-coarsened}.
We comment the above result obtained from a direct application of triangle inequality is indeed weaker than the result stated in Theorem~\ref{thm: coarsen12};
we will then devote the remaining part of this section to the proof thereof.
\end{remark}

\subsection{Proof of Theorem \ref{thm: coarsen12}}
\label{app: coarsen12}

For one side of the result, we can follow the derivation in Lemma~\ref{lem: decomp_GW} and apply Theorem~\ref{thm: self-coarsened} to have
\begin{align*}
\text{GW}_2^2(G_1, G_2) - \text{GW}_2^2(G_1^{(c)}, G_2^{(c)})
=& I_1 - I'_1 + I_2 - I'_2 + 2(I'_3 - I_3) \leq I_1 - I'_1 + I_2 - I'_2 \\
\leq& \lambda_{N_1-n_1+1} \sum_{i=1}^{n_1} \paren{\lambda_{i} - \lambda^{(c)}_{i}} + C_{\mtx U_1, n_1}
+ \lambda_{N_2-n_2+1} \sum_{i=1}^{n_2} \paren{\lambda_{i} - \lambda^{(c)}_{i}} + C_{\mtx U_2, n_2},
\end{align*}
where we recall $I_3 \defeq \Tr \paren{\mtx S_1 \mtx T^* \mtx S_2 \paren{\mtx T^*}^{\intercal}}$ and $I'_3$ now is $\Tr \paren{\mtx S_1^{(c)} \mtx T_{cc}^* \mtx S_2^{(c)} \paren{\mtx T_{cc}^*}^{\intercal}}$ ($T_{cc}^*$ is the optimal transport matrix for $G_1^{(c)}$ and $G_2^{(c)}$).

For the other side, we still decompose the object $\text{GW}_2^2(G_1^{(c)}, G_2^{(c)}) - \text{GW}_2^2(G_1, G_2)$ as $(I'_1 - I_1) + (I'_2 - I_2) + 2 (I_3 - I'_3)$, 
and similarly we can follow Lemma~\ref{lem: bd_diff_I1} to bound the first two terms.
The next task is to disassemble $I_3 - I'_3$.

We first prepare some notations for the analysis. 
To clarify the affiliation, we redefine $\mtx U_1 \defeq \mtx W_1^\frac12 \mtx{S}_1 \mtx W_1^\frac12$ with eigenvalues $\lambda_{1, 1} \geq \lambda_{1, 2} \geq \cdots \geq \lambda_{1, N_1}$, 
$\mtx V_1 = \mtx P \mtx W_2^{\frac12} \mtx S_2 \mtx W_2^{\frac12} \mtx P^\intercal$ with eigenvalues $\nu_{1,1} \geq \nu_{1,2}\geq \cdots \geq \nu_{1,N_1}$\footnote{We index $\mtx P \mtx W_2^{\frac12} \mtx S_2 \mtx W_2^{\frac12} \mtx P^\intercal$ as $\mtx V_1$ since it is associated with $\mtx U_1$ and is an $N_1 \times N_1$ matrix.}, 
$\mtx U_2 = \mtx W_2^\frac12 \mtx{S}_2 \mtx W_2^\frac12$ with eigenvalues $\lambda_{2, 1} \geq \lambda_{2, 2} \geq \cdots \geq \lambda_{2, N_2}$, 
$\mtx V_2 = \mtx P^\intercal \mtx W_1^{\frac12} \mtx S_1 \mtx W_1^{\frac12} \mtx P$ with eigenvalues $\nu_{2,1} \geq \nu_{2,2}\geq \cdots \geq \nu_{2,N_2}$,
and similarly re-introduce $C_{p, 1}, \bar{\mtx C}_{w, 1}, \mtx C_{w, 1}$, $C_{p, 2}, \bar{\mtx C}_{w, 2}, \mtx C_{w, 2}$, and
\begin{align*}
\mtx\Pi_{w, 1} \defeq \mtx W_1^\frac{1}{2} \mtx C_{p, 1}^\intercal \bar{\mtx C}_{w, 1}\mtx W_1^{-\frac{1}{2}} = \mtx C_{w, 1}^\intercal\mtx C_{w, 1}, \qquad 
\mtx\Pi_{w, 2} \defeq \mtx W_2^\frac{1}{2}\mtx C_{p, 2}^\intercal \bar{\mtx C}_{w, 2}\mtx W_2^{-\frac{1}{2}} = \mtx C_{w, 2}^\intercal\mtx C_{w, 2}.
\end{align*}
Recalling Lemma~\ref{lem: bd_diff_I3}, we can analogously obtain $I'_3 - I_3 \leq 0$; 
replacing $\mtx T_{cc}^*$ with $\mtx C_{p, 1} \mtx T^* \mtx C_{p, 2}^\intercal$, we have
\begin{align*}
I_3 - I'_3 
\leq &\Tr \paren{\mtx S_1 \mtx T^* \mtx S_2 \paren{\mtx T^*}^{\intercal}} - \Tr \paren{\mtx S_1^{(c)} \mtx C_{p, 1} \mtx T^* \mtx C_{p, 2}^\intercal \mtx S_2^{(c)} \mtx C_{p, 2} \paren{\mtx T^*}^\intercal \mtx C_{p, 1}^\intercal} \\
= &\Tr \paren{\mtx U_1 \mtx P \mtx U_2 \mtx P^{\intercal}} - \Tr \paren{\mtx\Pi_{w, 1} \mtx U_1 \mtx\Pi_{w, 1} \mtx P \mtx\Pi_{w, 2} \mtx U_2 \mtx\Pi_{w, 2} \mtx P^{\intercal}},
\end{align*}
where we apply the same derivation as in Equation~(\ref{eqn: introduce_proj}) to obtain the second line above.
For simplicity, we let $\mtx U_1^\Pi \defeq \mtx\Pi_{w, 1} \mtx U_1 \mtx\Pi_{w, 1}$ and $\mtx U_2^\Pi \defeq \mtx\Pi_{w, 2} \mtx U_2 \mtx\Pi_{w, 2}$.
We can now bound $I_3 - I'_3$ as
\begin{equation}
\label{eqn:coarsen12bd}
\begin{aligned}
I_3 - I'_3 
\leq& \Tr \paren{\mtx U_1 \mtx P \mtx U_2 \mtx P^{\intercal}} - \Tr \paren{\mtx U_1^\Pi \mtx P \mtx U_2^\Pi \mtx P^{\intercal}} \\
=& \Tr \paren{\mtx U_1 \mtx P \mtx U_2 \mtx P^{\intercal}} - \Tr \paren{\mtx U_1^\Pi \mtx P \mtx U_2 \mtx P^{\intercal}} + \Tr \paren{\mtx U_1^\Pi \mtx P \mtx U_2 \mtx P^{\intercal}} - \Tr \paren{\mtx U_1^\Pi \mtx P \mtx U_2^\Pi \mtx P^{\intercal}} \\
=& \Tr\brkt{\paren{\mtx U_1 - \mtx U_1^\Pi} \mtx P \mtx U_2 \mtx P^{\intercal}} + \Tr\brkt{\mtx U_1^\Pi \mtx P \paren{\mtx U_2 - \mtx U_2^\Pi} \mtx P^{\intercal}} \\
=& \Tr\brkt{\paren{\mtx U_1 - \mtx U_1^\Pi} \mtx P \mtx U_2 \mtx P^{\intercal}} + 
\Tr\brkt{\mtx U_1 \mtx P \paren{\mtx U_2 - \mtx U_2^\Pi} \mtx P^{\intercal}} \\
&\qquad - \Tr\brkt{\mtx U_1 \mtx P \paren{\mtx U_2 - \mtx U_2^\Pi} \mtx P^{\intercal}} +
\Tr\brkt{\mtx U_1^\Pi \mtx P \paren{\mtx U_2 - \mtx U_2^\Pi} \mtx P^{\intercal}} \\
=& \Tr\brkt{\paren{\mtx U_1 - \mtx U_1^\Pi} \mtx V_1} + \Tr\brkt{\mtx V_2 \paren{\mtx U_2 - \mtx U_2^\Pi}} - \Tr\brkt{\paren{\mtx U_1 - \mtx U_1^\Pi} \mtx P \paren{\mtx U_2 - \mtx U_2^\Pi} \mtx P^{\intercal}}.
\end{aligned} 
\end{equation}

The first two terms in the last line above can be directly addressed by Lemma~\ref{lem: bd_diff_I3_further};
for the last term, we can bound it as
\begin{align*}
\abs{\Tr\brkt{\paren{\mtx U_1 - \mtx U_1^\Pi} \mtx P \paren{\mtx U_2 - \mtx U_2^\Pi} \mtx P^{\intercal}}} 
\leq& \norm{\paren{\mtx U_1 - \mtx U_1^\Pi} \mtx P}_F \norm{\paren{\mtx U_2 - \mtx U_2^\Pi} \mtx P^{\intercal}}_F \\
\leq& \norm{\mtx U_1 - \mtx U_1^\Pi}_F \norm{\mtx P} \norm{\mtx U_2 - \mtx U_2^\Pi}_F \norm{\mtx P}
\leq \norm{\mtx U_1 - \mtx U_1^\Pi}_F \norm{\mtx U_2 - \mtx U_2^\Pi}_F,
\end{align*}
where Lemma~\ref{lem:normalized_transport} shows $\norm{\mtx P} \leq 1$ and justifies the last inequality.

We now give another useful fact that $I_1 - I'_1 = \norm{\mtx U_1 - \mtx U_1^\Pi}_F^2$:
\begin{align*}
I_1 - I'_1 \overset{(\rom 1)}{=}& \Tr \brkt{\mtx U_1^2 -\paren{\mtx U_1^\Pi}^2} = \Tr \brkt{\mtx U_1^2 + \paren{\mtx U_1^\Pi}^2} - 2 \Tr \brkt{\paren{\mtx\Pi_{w, 1} \mtx U_1 \mtx\Pi_{w, 1}}^2} \\
=& \Tr \brkt{\mtx U_1^2 + \paren{\mtx U_1^\Pi}^2} 
    - \Tr \paren{\mtx\Pi_{w, 1} \mtx U_1 \mtx\Pi_{w, 1} \mtx U_1}
    - \Tr \paren{\mtx U_1 \mtx\Pi_{w, 1} \mtx U_1 \mtx\Pi_{w, 1}} 
= \Tr \brkt{\paren{\mtx U_1 - \mtx U_1^\Pi}^2} \\
=& \norm{\mtx U_1 - \mtx U_1^\Pi}_F^2,
\end{align*}
where $(\rom 1)$ comes from Equation~(\ref{eqn:I1}).
Analogously we have $I_2 - I'_2 = \norm{\mtx U_2 - \mtx U_2^\Pi}_F^2$.
Combining the above pieces together we can bound the object as
\begin{align*}
\text{GW}_2^2(G_1^{(c)}, G_2^{(c)}) - \text{GW}_2^2(G_1, G_2) 
\leq& -\norm{\mtx U_1 - \mtx U_1^\Pi}_F^2 - \norm{\mtx U_2 - \mtx U_2^\Pi}_F^2 + \\
&\quad 2 \Tr\brkt{\paren{\mtx U_1 - \mtx U_1^\Pi} \mtx V_1} + 2 \Tr\brkt{\mtx V_2 \paren{\mtx U_2 - \mtx U_2^\Pi}} + 2 \norm{\mtx U_1 - \mtx U_1^\Pi}_F \norm{\mtx U_2 - \mtx U_2^\Pi}_F \\
\leq& 2 \Tr\brkt{\paren{\mtx U_1 - \mtx U_1^\Pi} \mtx V_1} + 2 \Tr\brkt{\mtx V_2 \paren{\mtx U_2 - \mtx U_2^\Pi}} \\
\overset{(\rom 1)}{\leq}& 2 \cdot \left[ \nu_{1, N_1 - n_1 + 1} \sum_{i=1}^{n_1} \paren{\lambda_{1, i} - \lambda^{(c)}_{1, i}} + C_{\mtx U_1, \mtx V_1, n_1} + \right.\\
&\quad \left. \nu_{2, N_2 - n_2 + 1} \sum_{i=1}^{n_2} \paren{\lambda_{2, i} - \lambda^{(c)}_{2, i}} + C_{\mtx U_2, \mtx V_2, n_2} \right],
\end{align*}
where we reuse Inequality~(\ref{eqn:I1}) to attain $(\rom 1)$.
The proof of Theorem~\ref{thm: coarsen12} is then completed.

\subsection{Proof of some technical results}\label{app: proof_in_main_thm}
\subsubsection{Proof of Lemma \ref{lem: decomp_GW}}
\begin{proof}
Following the definition in \Cref{eqn:gwd}, we rewrite the GW distance in the trace form and have
\begin{align}
\dotp{\mtx M}{\mtx T^*} &= \dotp{f_{1}(\mtx S_1) \mtx m_1 \mtx{1}_{N_{2}}^{\intercal}}{\mtx T^*} + 
\dotp{\mtx{1}_{N_{1}} \mtx m_2^{\intercal} f_{2}(\mtx S_2)^{\intercal}}{\mtx T^*} - 
\dotp{h_{1}(\mtx S_1) \mtx T^* h_{2}(\mtx S_2)^{\intercal}}{\mtx T^*} \nonumber \\
&= \Tr \paren{f_{1}(\mtx S_1) \mtx m_1 \mtx{1}_{N_{2}}^{\intercal} \paren{\mtx T^*}^{\intercal}} + 
\Tr \paren{f_{2}(\mtx S_2) \mtx m_2 \mtx{1}_{N_{1}}^{\intercal} \mtx T^*} -
\Tr \paren{h_{1}(\mtx S_1) \mtx T^* h_{2}(\mtx S_2)^{\intercal} \paren{\mtx T^*}^{\intercal}} \nonumber \\
&= \Tr \paren{\mtx m_1^{\intercal} f_{1}(\mtx S_1) \mtx m_1} + 
\Tr \paren{\mtx m_2^{\intercal} f_{2}(\mtx S_2) \mtx m_2} -
\Tr \paren{h_{1}(\mtx S_1) \mtx T^* h_{2}(\mtx S_2)^{\intercal} \paren{\mtx T^*}^{\intercal}};
\label{eqn:expansion}
\end{align}
the third equation above holds because for any $\mtx T \in \Pi(\mu_1, \mu_2)$, $\mtx T \mtx{1}_{N_{2}} = \mtx m_1$ and $\mtx T^{\intercal} \mtx{1}_{N_{1}} = \mtx m_2$.

In the classical square loss case, we can immediately have $f_1(\mtx S_1) = \mtx{S_1} \odot \mtx{S_1}, f_2(\mtx S_2) = \mtx{S_2} \odot \mtx{S_2}, h_{1}(\mtx S_1) = \mtx S_1$ and $h_{2}(\mtx S_2) = 2 \mtx S_2$,
where $\odot$ denotes the Hadamard product of two matrices with the same size.
We can accordingly expand the first term in \Cref{eqn:expansion} as
\begin{align*}
\Tr \paren{\mtx m_1^{\intercal} f_{1}(\mtx S_1) \mtx m_1} 
= \mtx m_1^{\intercal} \paren{\mtx{S_1} \odot \mtx{S_1}} \mtx m_1
= \Tr \paren{\mtx{S_1}^{\intercal} \text{diag}(\mtx{m_1}) \mtx{S_1} \text{diag}(\mtx{m_1})},
\end{align*}
in which the proof of the last equation is provided in a summary sheet by \citet{minka2000old}. 
We note $\mtx W_1 \defeq \text{diag}(\mtx{m_1})$ and $\mtx S_1$ is constructed symmetric; 
combining the pieces above we have
\begin{align*}
\Tr \paren{\mtx m_1^{\intercal} f_{1}(\mtx S_1) \mtx m_1} = 
\Tr \paren{\paren{\mtx W_1^\frac12 \mtx{S_1} \mtx W_1^\frac12} \paren{\mtx W_1^\frac12 \mtx{S_1} \mtx W_1^\frac12}},
\end{align*}
and similarly we obtain $\Tr \paren{\mtx m_2^{\intercal} f_{2}(\mtx S_2) \mtx m_2} = \Tr \paren{\paren{\mtx W_2^\frac12 \mtx{S_2} \mtx W_2^\frac12} \paren{\mtx W_2^\frac12 \mtx{S_2} \mtx W_2^\frac12}}$.
The GW distance (\ref{eqn:expansion}) can be therefore represented as
\begin{align}
\label{eqn:tr_gw}
\underbrace{\Tr \paren{\paren{\mtx W_1^\frac12 \mtx{S_1} \mtx W_1^\frac12} \paren{\mtx W_1^\frac12 \mtx{S_1} \mtx W_1^\frac12}}}_{=\,:I_1} + 
\underbrace{\Tr \paren{\paren{\mtx W_2^\frac12 \mtx{S_2} \mtx W_2^\frac12} \paren{\mtx W_2^\frac12 \mtx{S_2} \mtx W_2^\frac12}}}_{=\,: I_2} - 
2 \underbrace{\Tr \paren{\mtx S_1 \mtx T^* \mtx S_2 \paren{\mtx T^*}^{\intercal}}}_{=\,:I_3}.
\end{align}
\end{proof}

\subsubsection{Proof of Lemma \ref{lem: bd_diff_I1}}
\begin{proof}
First, recall that $\mtx\Pi_{w} = \mtx W_1^\frac{1}{2}\mtx C_{p}^\intercal\bar{\mtx C}_{w}\mtx W_1^{-\frac{1}{2}} = \mtx C_{w}^\intercal\mtx C_{w}$. So, we have
\begin{equation}
\label{eqn: introduce_proj}
\begin{aligned}
    I_1' &= \Tr \paren{\brkt{\paren{\mtx C_{p}\mtx W_1\mtx C_{p}^\intercal}^\frac{1}{2}\paren{\bar{\mtx C}_{w}\mtx S_1\bar{\mtx C}_{w}^\intercal}\paren{\mtx C_{w}\mtx W_1\mtx C_{w}^\intercal}^\frac{1}{2}}^2}\\
    &= \Tr\brkt{{\paren{\mtx C_{p}\mtx W_1\mtx C_{p}^\intercal}\paren{\bar{\mtx C}_{w}\mtx S_1\bar{\mtx C}_{w}^\intercal}\paren{\mtx C_{w}\mtx W_1\mtx C_{w}^\intercal}\paren{\bar{\mtx C}_{w}\mtx S_1\bar{\mtx C}_{w}^\intercal}}}\\
    &= \Tr \paren{\mtx W_1^\frac{1}{2}\mtx C_{p}^\intercal\bar{\mtx C}_{w}\mtx S_1\bar{\mtx C}_{w}^\intercal\mtx C_{w}\mtx W_1\mtx C_{p}^\intercal\bar{\mtx C}_{w}\mtx S_1\bar{\mtx C}_{w}^\intercal\mtx C_{p}\mtx W_1^\frac{1}{2}}\\
    &= \Tr\paren{\mtx\Pi_{w}\mtx U_1\mtx\Pi_{w}^\intercal\mtx\Pi_{w}\mtx U_1\mtx\Pi_{w}}\\
    &= \Tr\brkt{\paren{\mtx\Pi_{w}\mtx U_1\mtx\Pi_{w}}^2}.
\end{aligned}
\end{equation}

This implies $I_1 - I_1' = \Tr \brkt{\mtx U^2 - \paren{\mtx\Pi_{w} \mtx U \mtx \Pi_{w}}^2}$. Applying Lemma~\ref{thm:trace_bound} yields $I_1 - I_1' \geq 0$.

To bound the other direction, we have
\begin{align}
\label{eqn:I1}
&\quad\, I_1 - I'_1
= \Tr \brkt{\mtx U^2 - \paren{\mtx\Pi_{w} \mtx U \mtx \Pi_{w}}^2} \nonumber \\
&= \sum_{i=1}^{N_1} \lambda_{i}^2 - \sum_{i=1}^{n_1} \paren{\lambda^{(c)}_{i}}^2 
\stackrel{\ri}{\leq} \sum_{i=1}^{N_i} \lambda_{i}^2 - \sum_{i=1}^{n_i} \lambda^{(c)}_{i} \lambda_{N-n+i} \nonumber \\
&= \sum_{i=1}^{n_1} \paren{\lambda_{i} - \lambda^{(c)}_{i}} \lambda_{N-n+i} + \sum_{i=1}^{n_1} \lambda_{i} (\lambda_{i} - \lambda_{N-n+i}) + \!\!\sum_{i=n_1+1}^{N_1} \lambda_{i}^2 \nonumber \\
&= \sum_{i=1}^{n_1} \paren{\lambda_{i} - \lambda^{(c)}_{i}} \lambda_{N_1-n_1+i} + C_{\mtx U, n_1} \nonumber \\
&\stackrel{\rii}{\leq} \lambda_{N_1-n_1+1} \sum_{i=1}^{n_1} \paren{\lambda_{i} - \lambda^{(c)}_{i}} + C_{\mtx U, n_1}, 
\end{align}
Here, both (i) and (ii) are by Theorem \ref{thm: poincare_sep}. 
\end{proof}

\subsubsection{Proof of Lemma \ref{lem: bd_diff_I3_further}}
\begin{proof}
By applying Lemma \ref{lem: bd_diff_I3}, we have
\begin{align*}
0\leq \Tr \paren{\mtx S_1 \mtx T^* \mtx S_2 \paren{\mtx T^*}^{\intercal} - \mtx S_1^{(c)} \mtx T_{co}^* \mtx S_2 \paren{\mtx T_{co}^*}^{\intercal}} 
&\leq \Tr \brkt{\paren{\mtx U - \mtx \Pi_{w} \mtx U \mtx \Pi_{w}} \mtx V}\\
&= \Tr \paren{\mtx U \mtx V} - \Tr \paren{\mtx C_{w} \mtx U \mtx C_{w}^\intercal \mtx C_{w} \mtx V \mtx C_{w}^\intercal}.
\end{align*}
Recall that $\set{\lambda^{(c)}_{i}}_{i=1}^{n_1}$ are eigenvalues of $\mtx C_{w}\mtx U\mtx C_{w}^\intercal$, and let $\nu_{1} \geq \nu_{2}\geq \cdots \geq \nu_{n_1}$ be eigenvalues of $\mtx C_{w}\mtx V\mtx C_{w}^\intercal$.
Applying Ruhe's trace inequality (c.f.\ Appendix~\ref{sec:von_neumann}), we further have
\begin{align}
&\Tr \paren{\mtx U \mtx V} - \Tr \paren{\mtx C_{w} \mtx U \mtx C_{w}^\intercal \mtx C_{w} \mtx V \mtx C_{w}^\intercal}
\stackrel{\ri}{\leq} \sum_{i=1}^{N_1} \lambda_{i} \nu_{i} - \sum_{i=1}^{n_1} \lambda^{(c)}_{i} \nu^{(c)}_{n_1-i+1}
\stackrel{\rii}{\leq} \sum_{i=1}^{N_1} \lambda_{i} \nu_{i} - \sum_{i=1}^{n_1} \lambda^{(c)}_{i} \nu_{N_1-i+1} \nonumber \\
&= \sum_{i=1}^{n_1} \paren{\lambda_{i} - \lambda^{(c)}_{i}} \nu_{N_1-i+1} + \sum_{i=1}^{n_1} \lambda_{i} \paren{\nu_{i} - \nu_{N_1-i+1}} + \sum_{i=n_1+1}^{N_1} \lambda_{i} \nu_{i}
= \sum_{i=1}^{n_1} \paren{\lambda_{i} - \lambda^{(c)}_{i}} \nu_{N_1-i+1} + C_{\mtx U, \mtx V, n_1} \nonumber \\
&\stackrel{\riii}{\leq} \nu_{N_1-n_1+1} \sum_{i=1}^{n_1} \paren{\lambda_{i} - \lambda^{(c)}_{i}} + C_{\mtx U, \mtx V, n_1},
\label{eqn:I3}
\end{align}
(i) is by Ruhe's trace inequality (Lemma \ref{lem: Ruhe_ineq}), since both $\mtx U$ and $\mtx V$ are PSD, and both (ii) and (iii) are by Poincar\'e separation theorem (Theorem \ref{thm: poincare_sep}).
\end{proof}

\subsection{Other technical lemmas}
\label{sec:normalized_transport}

\begin{lemma}\label{lem: bd_diff_I3}
We have
\begin{align*}
    0 \leq \Tr \paren{\mtx S_1 \mtx T^* \mtx S_2 \paren{\mtx T^*}^{\intercal} - \mtx S_1^{(c)} \mtx T_{co}^* \mtx S_2 \paren{\mtx T_{co}^*}^{\intercal}} \leq \Tr\brkt{\paren{\mtx U - \mtx\Pi_{w}\mtx U\mtx\Pi_{w}}\mtx V}.
\end{align*}
\end{lemma}
\begin{proof}
The proof is based on the optimality of $\mtx T_{co}^*$, i.e. the GW distance induced by $\mtx T_{co}^*$ must be upper bounded by any other transport matrix.
Intuitively, we can imagine the mass of a cluster center is transported to the same target nodes in $G_2$ as the original source nodes within this cluster, which corresponds to the transport matrix $\wt{\mtx T}_{co} \defeq \mtx C_{1,p} \mtx T^\ast$.
We can verify $\widetilde{\mtx T}_c\in\Pi(\mu_1^{(c)}, \mu_2)$ is feasible, since $(\mtx C_{1, p}\mtx T^\ast)\mtx 1_{N_2} = \mtx C_{1, p}\mtx m_1 = \mtx m_1^{(c)}$ and $\paren{\mtx C_{1,p} \mtx T^*}^\intercal \mtx 1_{n_1} = (\mtx T^\ast)^\intercal \mtx 1_{N_1} = \mtx m_2$.

To derive the upper bound, applying the optimality of $\mtx T_{co}^\ast$ yields
\begin{align}
&\quad\,\Tr \paren{\mtx S_1 \mtx T^* \mtx S_2 \paren{\mtx T^*}^{\intercal} - \mtx S_1^{(c)} \mtx T_{co}^* \mtx S_2 \paren{\mtx T_{co}^*}^{\intercal}}
\leq \Tr \paren{\mtx S_1 \mtx T^* \mtx S_2 \paren{\mtx T^*}^{\intercal}} - \Tr \paren{\mtx S_1^{(c)} \wt{\mtx T}_{co} \mtx S_2 \wt{\mtx T}_{co}^{\intercal}}\\
&= \Tr\paren{\mtx S_1 \mtx T^* \mtx S_2 \paren{\mtx T^*}^{\intercal}} - \Tr \paren{\big(\bar{\mtx C}_{w}\mtx S_1\bar{\mtx C}_{w}^\intercal\big)\big(\mtx C_{p}\mtx T^\ast\big)\mtx S_2\big(\mtx C_{p}\mtx T^\ast\big)^\intercal}\\
&= \Tr\paren{\mtx S_1\mtx T^\ast\mtx S_2(\mtx T^\ast)^\intercal} - \Tr\paren{\bar{\mtx C}_{w}\mtx S_1\bar{\mtx C}_{w}^\intercal\mtx C_{p}\mtx T^\ast\mtx S_2(\mtx T^\ast)^\intercal\mtx C_{p}^\intercal}\\
&= \Tr\paren{\mtx S_1\big(\mtx W_1^\frac{1}{2}\mtx P\mtx W_2^\frac{1}{2}\big)\mtx S_2\big(\mtx W_1^\frac{1}{2}\mtx P\mtx W_2^\frac{1}{2}\big)^\intercal} - \Tr \paren{\bar{\mtx C}_{w}\mtx{S}_1 \bar{\mtx C}_{w}^\intercal \mtx C_{p} \big(\mtx W_1^{\frac12} \mtx P \mtx W_2^{\frac12}\big) \mtx S_2 \big(\mtx W_2^{\frac12} \mtx P^\intercal \mtx W_1^{\frac12}\big) \mtx C_{p}^\intercal} \nonumber \\
&= \Tr \paren{\mtx{S}_1 \mtx W_1^{\frac12} \mtx P \mtx W_2^{\frac12} \mtx S_2 \mtx W_2^{\frac12} \mtx P^\intercal \mtx W_1^{\frac12}} - 
\Tr \paren{\bar{\mtx C}_{w} \mtx{S}_1 \bar{\mtx C}_{w}^\intercal \mtx C_{p} \mtx W_1^{\frac12} \mtx P \mtx W_2^{\frac12} \mtx S_2 \mtx W_2^{\frac12} \mtx P^\intercal \mtx W_1^{\frac12} \mtx C_{p}^\intercal} \nonumber \\
&= \Tr \brkt{\paren{
\underbrace{\mtx W_1^{\frac12} \mtx{S}_1 \mtx W_1^{\frac12} - \mtx W_1^{\frac12} \mtx C_{p}^\intercal \bar{\mtx C}_{w} \mtx{S}_1 \bar{\mtx C}_{w}^\intercal \mtx C_{p} \mtx W_1^{\frac12}}_{\mtx D_1}} \mtx P \mtx W_2^{\frac12} \mtx S_2 \mtx W_2^{\frac12} \mtx P^\intercal}.
\label{eqn:upside_bd}
\end{align}
The treatment is similar for the lower bound. We replace $\mtx T^*$ with $\wt{\mtx T} \defeq \bar{\mtx C}_{w}^\intercal \mtx T_{co}^* \in \Pi(\mu_1, \mu_2)$.
Then again due to the optimality of $\mtx T^*$:
\begin{align*}
\Tr \paren{\mtx S_1^{(c)} \mtx T_{co}^* \mtx S_2 \paren{\mtx T_{co}^*}^{\intercal} - \mtx S_1 \mtx T^* \mtx S_2 \paren{\mtx T^*}^{\intercal}} \leq \Tr\paren{\mtx S_1^{(c)}\mtx T_{co}^\ast\mtx S_2(\mtx T_{co}^\ast)^\intercal} - \Tr\paren{\mtx S_1\widetilde{\mtx T}\mtx S_2\widetilde{\mtx T}^\intercal} = 0
\end{align*}
given our definition. 
\end{proof}
\textbf{Remark.}
An intuitive scheme to control the upper bound (\ref{eqn:upside_bd}) is to upper bound the trace of the $G_1$-related matrix difference $\mtx D_1$, since in coarsening $G_1$ we have no information about $G_2$;
we will shortly showcase the term is the key to bound the whole GW distance difference $\abs{(\ref{eqn:tr_gw}) - (\ref{eqn:tr_gwc})}$ as well.

\begin{lemma}
\label{lem:normalized_transport}
Consider a non-negative matrix $\mtx T \in \mb R^{N_1 \times N_2}$ in which all the elements $t_{ij} \geq 0$. 
We keep to denote $\mtx W_1 = \mathrm{diag}\paren{\mtx T \mtx 1_{N_2}}, \mtx W_2 = \mathrm{diag}\paren{\mtx T^\intercal \mtx 1_{N_1}}$,
and $\mtx P = \mtx W_1^{-\frac12} \mtx T \mtx W_2^{-\frac12}$.
Then we have $\|\mtx P\| \leq 1$.
\end{lemma}
\begin{proof}
Motivated by the regular proof for bounding the eigenvalues of normalized Laplacian matrices,
with two arbitrary vectors $\mtx u \in \mb R^{N_1}, \mtx v \in \mb R^{N_2}$ on the unit spheres we can recast the target statement as
\begin{align}
\label{eqn:norm_trans}
1 - \mtx u^\intercal \mtx P \mtx v \geq 0, \quad \forall~\mtx u, \mtx v \text{~s.t.~} \|\mtx u\|=1, \|\mtx v\|=1.
\end{align}
We denote the diagonals for $\mtx W_1$ and $\mtx W_2$ respectively as $\mtx m_1 = [m_{1, 1}, m_{1, 2}, \dots, m_{1, N_1}]^{\intercal}$ and $\mtx m_2 = [m_{2, 1}, m_{2, 2}, \dots, m_{2, N_2}]^{\intercal}$;
then we can rewrite the right-hand-side quantity in Inequality~(\ref{eqn:norm_trans}) as
\begin{align*}
1 - \mtx u^\intercal \mtx P \mtx v &= \frac12 \paren{\|\mtx u\|^2 + \|\mtx v\|^2} - \sum_{i=1}^{N_1} \sum_{j=1}^{N_2} \frac{t_{ij} \mtx u_i \mtx v_j}{\sqrt{m_{1, i} m_{2, j}}} \\
&= \frac12 \sum_{i=1}^{N_1} \sum_{j=1}^{N_2} \frac{t_{ij} \mtx u_i^2}{m_{1, i}}
+ \frac12 \sum_{j=1}^{N_2} \sum_{i=1}^{N_1} \frac{t_{ij} \mtx v_j^2}{m_{2, j}}
- \sum_{i=1}^{N_1} \sum_{j=1}^{N_2} \frac{t_{ij} \mtx u_i \mtx v_j}{\sqrt{m_{1, i} m_{2, j}}} \\
&= \frac12 \sum_{i=1}^{N_1} \sum_{j=1}^{N_2} t_{ij} \paren{\frac{\mtx u_i}{\sqrt{m_{1, i}}} - \frac{\mtx v_j}{\sqrt{m_{2, j}}}}^2 \geq 0,
\end{align*}
where the second equation holds due to the conditions $\sum_{j=1}^{N_2} t_{ij} = m_{1, i}$ and $\sum_{i=1}^{N_1} t_{ij} = m_{2, j}$,
and the last inequality holds since $t_{ij} \geq 0, \forall i \in [N_1], j \in [N_2]$.
\end{proof}

\begin{lemma}
\label{thm:trace_bound}
Consider a positive semi-definite (PSD) similarity matrix $S \in \mb R^{N \times N}$ along with the probability mass vector $\mtx m$ and the diagonal matrix $\mtx W \defeq \mathrm{diag}\paren{\mtx m}$.
For any non-overlapping partition $\set{\m P_1, \m P_2, \dots, \m P_n}$, we denote the corresponding coarsening matrices $\mtx C_p, \bar{\mtx C}_w$ and the projection matrix $\mtx \Pi_w = \mtx C_w^\intercal \mtx C_w$.
Let $\mtx A \defeq \mtx W^\frac12 \mtx{S} \mtx W^\frac12$. Then we have
\begin{align}
\Tr \paren{\mtx A^2}
\geq \Tr \brkt{\paren{\mtx \Pi_w \mtx A \mtx \Pi_w}^2}.
\label{eqn:trace_bound}
\end{align}
\end{lemma}

\begin{proof}
We first transform $\Tr \brkt{\paren{\mtx\Pi_w \mtx A \mtx \Pi_w}^2}$ as follows:
\begin{align*}
\Tr \brkt{\paren{\mtx\Pi_w \mtx A \mtx \Pi_w}^2} = \Tr\paren{\mtx\Pi_w \mtx A \mtx \Pi_w \mtx\Pi_w \mtx A \mtx \Pi_w} = \Tr\paren{\mtx C_w^\intercal \mtx C_w \mtx A \mtx C_w^\intercal \mtx C_w \mtx A \mtx C_w^\intercal \mtx C_w}
= \Tr\paren{\mtx C_w \mtx A \mtx C_w^\intercal \mtx C_w \mtx A \mtx C_w^\intercal},
\end{align*}
and the last two equations hold due to $\mtx C_w \mtx C_w^\intercal = \mtx I_n$.

We notice $\mtx A \defeq \mtx W^\frac12 \mtx{S} \mtx W^\frac12$ is symmetric and we can apply Poincar\'e separation theorem (c.f.\ Appendix~\ref{sec:interlacing} for the complete statement) to control the eigenvalues of $\mtx C_w \mtx A \mtx C_w^\intercal$.
Specifically, let $\lambda_1 \geq \lambda_2 \geq \dots \geq \lambda_N $ be the eigenvalues of $\mtx A$, and let $\lambda^{(c)}_1 \geq \dots \geq \lambda^{(c)}_n$ be the eigenvalues of $\mtx C_w \mtx A \mtx C_w^\intercal$; 
for all $i \in [n]$ we have $\lambda_i \geq \lambda^{(c)}_i \geq 0$ (being non-negative due to the PSD-ness of $\mtx A$); 
and a further conclusion $\lambda_i^2 \geq \paren{\lambda^{(c)}_i}^2, \forall i \in [n]$.
We can therefore complete the proof with
$\Tr \paren{\mtx A^2} = \sum_{i=1}^N \lambda_i^2 \geq \sum_{i=1}^n \paren{\lambda^{(c)}_i}^2 = \Tr \brkt{\paren{\mtx\Pi_w \mtx A \mtx \Pi_w}^2}$.
\end{proof}

\end{document}